\documentclass{article}

\usepackage{microtype}
\usepackage{graphicx}
\usepackage{subcaption}
\usepackage{booktabs} 

\usepackage[preprint]{icml2026}

\usepackage[utf8]{inputenc} 
\usepackage[T1]{fontenc}    
\usepackage[colorlinks=true,citecolor=teal]{hyperref}
\usepackage{url}            
\usepackage{booktabs}       
\usepackage{amsfonts}       
\usepackage{nicefrac}       
\usepackage{microtype}      
\usepackage{xcolor}         

\usepackage{amsmath}
\usepackage{amssymb}
\usepackage{mathtools}
\usepackage{amsthm}
\usepackage{multirow}
\usepackage{float}
\usepackage{algorithm}
\usepackage{algorithmic}
\usepackage{makecell}

\usepackage{wrapfig}
\usepackage{adjustbox}
\usepackage{enumitem}
\usepackage{graphicx}
\usepackage{bm}
\usepackage{comment}

\usepackage[capitalize,noabbrev]{cleveref}

\usepackage[utf8]{inputenc} 
\usepackage[T1]{fontenc}    
\usepackage[colorlinks=true,citecolor=teal]{hyperref}
\usepackage{url}            
\usepackage{booktabs}       
\usepackage{amsfonts}       
\usepackage{nicefrac}       
\usepackage{microtype}      
\usepackage{xcolor}         

\usepackage{amsmath}
\usepackage{amssymb}
\usepackage{mathtools}
\usepackage{amsthm}
\usepackage{multirow}
\usepackage{float}
\usepackage{algorithm}
\usepackage{algorithmic}

\usepackage{wrapfig}
\usepackage{adjustbox}
\usepackage{enumitem}
\usepackage{graphicx}
\usepackage{bm}
\usepackage{comment}

\newtheorem{theorem}{Theorem}
\newtheorem{lemma}{Lemma}

\newtheorem{definition}{Definition}
\newtheorem{remark}{Remark}

\definecolor{softgreen}{RGB}{100,200,100}
\definecolor{lightblue}{RGB}{173,216,230}
\definecolor{lavender}{RGB}{230,230,250}

\definecolor{metacolor}{HTML}{0064E0}
\hypersetup{
  colorlinks=true,
  linkcolor=fireenginered,
  citecolor=metacolor,
  urlcolor=metacolor
}
\definecolor{fireenginered}{rgb}{0.81, 0.09, 0.13}

\usepackage{titletoc}
\newcommand\DoToC{%
  \startcontents
  \printcontents{}{1}{\textbf{Contents of Appendix}\vskip3pt\hrule\vskip5pt}
  \vskip3pt\hrule\vskip5pt
}

\makeatletter
\newcommand{\btriangle}{\mathpalette\btriangle@\relax}
\newcommand{\btriangle@}[2]{%
  \begingroup
  \sbox\z@{$\m@th#1\triangle$}%
  \makebox[\wd\z@]{%
    \raisebox{0.04\height}{%
      \resizebox{1.1\wd\z@}{0.96\ht\z@}{%
        $\m@th#1\blacktriangle$%
      }%
    }%
  }%
  \endgroup
}

\usepackage[textsize=tiny]{todonotes}

\icmltitlerunning{MetaMind: General and Cognitive World Models in Multi-Agent Systems by Meta-Theory of Mind}

\begin{document}

\twocolumn[
  \icmltitle{MetaMind: General and Cognitive World Models in Multi-Agent Systems \\ by Meta-Theory of Mind}


  \begin{icmlauthorlist}
    \icmlauthor{Lingyi Wang}{a1}
    \icmlauthor{Rashed Shelim}{a1}
    \icmlauthor{Walid Saad}{a1}
    \icmlauthor{Naren Ramakrishna}{a2}
  \end{icmlauthorlist}

  \icmlaffiliation{a1}{Department of Electrical and Computer Engineering, Virginia Tech, USA}
  \icmlaffiliation{a2}{Department of Computer Science, Virginia Tech, USA}

  \icmlcorrespondingauthor{Lingyi Wang}{lingyiwang@vt.edu}


  \vskip 0.3in
]



\printAffiliationsAndNotice{}  

\begin{abstract}
\vspace{-0.1cm}
A major challenge for world models in multi-agent systems is to understand interdependent agent dynamics, predict interactive multi-agent trajectories, and plan over long horizons with collective awareness, without centralized supervision or explicit communication.
In this paper, MetaMind, a general and cognitive world model for multi-agent systems that leverages a novel meta-theory of mind (Meta-ToM) framework, is proposed. Through MetaMind, each agent learns not only to predict and plan over its own beliefs, but also to inversely reason goals and beliefs from its own behavior trajectories. This self-reflective, bidirectional inference loop enables each agent to learn a metacognitive ability in a self-supervised manner. Then, MetaMind is shown to generalize the metacognitive ability from first-person to third-person through analogical reasoning. Thus, in multi-agent systems, each agent with MetaMind can actively reason about goals and beliefs of other agents from limited, observable behavior trajectories in a zero-shot manner, and then adapt to emergent collective intention without an explicit communication mechanism. Extended simulation results on diverse multi-agent tasks demonstrate that MetaMind can achieve superior task performance and outperform baselines in few-shot multi-agent generalization.
\end{abstract}

\vspace{-0.5cm}
\section{Introduction}
\vspace{-0.1cm}
World models endow intelligent agents with the human-like ability to jointly understand environment dynamics, predict future trajectories, and plan actions over an extended horizon \citep{ding2024understanding,hafner2025mastering,prasanna2024dreaming,wang2025world}. Particularly, by abstracting low-level, high-dimensional sensory data into compact state representations in a latent space, world models enable sample-efficient imagination and long-horizon reasoning through internal simulation rather than direct interactions with the environment \citep{vafa2024evaluating,sun2024learning}. These properties facilitate generalization ability, counterfactual reasoning, and long-term planning \citep{yang2025driving,wang2025disentangled}. Recent advances have demonstrated that world models can capture the underlying dynamics of complex domains, including visual control, open-world exploration, and physical reasoning by only learning from images or videos \citep{assran2025v,samsami2024mastering,wang2025world}. These advances enable a transition from reactive policies to model-based planning, thus providing a foundation for more general-purpose intelligence. 

\begin{figure*}[!t]
    \centering
    \includegraphics[width=0.9\linewidth]{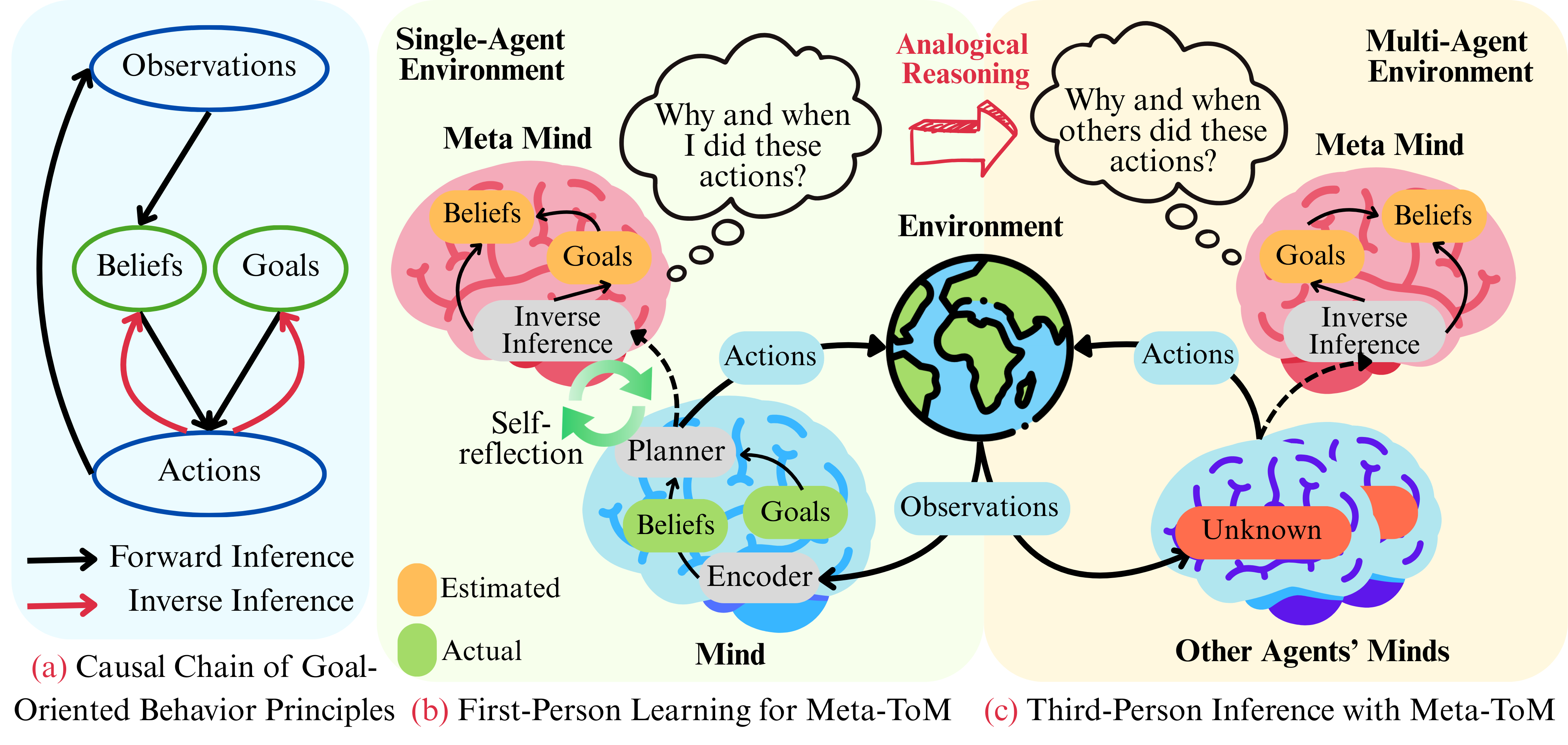}
    \vspace{-0.35cm}
    \caption{The proposed general and cognitive MetaMind in multi-agent systems.}
    \label{fig:sys}
    \vspace{-0.7cm}
\end{figure*}
\vspace{-0.15cm}
Despite the past success in single-agent systems, existing world models \citep{zhang2025revisiting,zhang2024decentralized,nomura2025decentralized} cannot support strategic planning in multi-agent systems due to non-stationary dynamics with adaptive behaviors, heterogeneous goals, and partial observations. Existing multi-agent world models can be divided into two categories. The first one is centralized training with decentralized execution (CTDE) models \citep{zhang2024decentralized} that rely on centralized representation aggregation to capture inter-agent dependencies during training. However, such a global context is unavailable under decentralized execution, which makes it hard to reuse the learned world model as an efficient test-time planner for long-horizon imagination. The second one is communication-based decentralized models \citep{nomura2025decentralized} that reduce partial observability via message exchange. However, messages are not directly available in counterfactual imagination and must be simulated, which increases inference cost and can exacerbate rollout drift. Hence, to enable world models for joint understanding, prediction, and long-horizon planning in multi-agent systems, we must overcome key challenges:
\begin{itemize}
  \vspace{-0.4cm}
    \item \emph{Understanding}: In multi-agent environments, the evolution of observations is shaped by \emph{coupled} interactions among multiple decision-makers, where each agent simultaneously acts on and reacts to the others. The same observation change can be consistent with different multi-agent explanations, thus making it fundamentally ambiguous to determine causal effects of each agent. Moreover, agents may have distinct goals, observe asymmetric information and follow heterogeneous policies \citep{zhang2024combo}, thus producing non-stationary, interdependent behavior that violates  the independence assumptions commonly used in standard world models. Hence, a major challenge is to understand \emph{agent-indexed, factorized} beliefs that disentangle the latent goals and beliefs of each participant while preserving their interaction structures. Such beliefs must attribute causal responsibility across agents through inverse inference rather than compressing everything in sensory data. In short, understanding requires structured, transferable abstractions that separate ``who acts,'' ``with which goal,'' and ``under which case,'' rather than a single-perspective dynamics model.
    
    \vspace{-0.2cm}
    \item \emph{Prediction}: Prediction in multi-agent systems is interactive, where each agent's future states are conditioned not only on the environment, but also on behavior of others under heterogeneous policies. This can lead to closed-loop distribution shift, long-horizon error compounding, and severe causal confounding when modeling trajectories from the perspective of a single agent. Such effects cannot be reliably handled by standard supervised dynamics predictors developed for single-agent world models. Hence, reliable prediction by world models in multi-agent systems requires cognitive inference over latent goals and beliefs, and modeling updates of other agents over multiple steps.
    
    \vspace{-0.2cm}
    \item \emph{Planning}: Planning is fundamentally strategic over joint action and belief spaces of multiple adaptive decision-makers. Rather than optimizing under a fixed environment, an agent must anticipate how others will interpret its behavior and adjust their responses for optimal planning. This suggests higher-order reasoning about goals and beliefs, e.g., “I know that you know that I know”, and more importantly, for maintaining and updating \emph{collective beliefs} over latent goals, private beliefs, and shared environment state as the interaction unfolds. Hence, it is challenging to achieve online long-horizon planning.
\end{itemize}

\vspace{-0.45cm}
To address these challenges, the world model at any given agent must treat other agents not as part of an unstructured background dynamics, but as \emph{intentional entities}, whose behavior is driven by goals and beliefs, and therefore admits causal, explainable principles as shown in Figure~\ref{fig:sys}(a). The theory-of-mind (ToM) provides this perspective by explicitly inferring others' beliefs and goals, thus forming a foundation for joint understanding, prediction, and planning in multi-agent systems \citep{10679907}. However, existing ToM-based approaches \citep{shum2019theory,westby2023collective,li2024long} typically rely on centralized or supervised training with specific partners and tasks, and learn a fixed perspective and structure of multi-agent interactions, which limits the generalization and cannot capture goal-conditioned causal mechanisms.

\vspace{-0.15cm}
The main contribution of this paper is \textit{\textbf{MetaMind}}, a general, cognitive world model that learns from an ego agent's experience and generalizes to multi-agent systems in a zero-shot manner. At the core of MetaMind, we propose a self-supervised \emph{Meta-ToM} framework that leverages only observable behavioral trajectories, without relying on explicit communication. Meta-ToM enables MetaMind to: (a) perform self-reflection by inversely inferring beliefs and goals to ensure inference-consistency, and turn passive dynamics modeling into goal-directed causal modeling for disentangled understanding, (b) generalize the inverse inference ability to unseen agents through analogical reasoning, which we name ``\emph{metacognition ability}'', for more reliable, explicit multi-agent trajectory predictions, and (c) learn collective beliefs for long-term, adaptive planning. These features provide a decentralized learning approach without symbolic communication systems or centralized context. Extended simulations across SMAC tasks show that MetaMind achieves ~54\% higher win rate than exsiting decentralized multi-agent world models under limited training steps, and a 2.4-fold improvement in few-shot generalization than exsiting centralized multi-agent world models.

\vspace{-0.4cm}
\section{Related Works}
\vspace{-0.1cm}
\subsection{World Models in Multi-Agent Systems}
\vspace{-0.1cm}
Recent advances in multi-agent systems have increasingly explored the use of world models to improve sample efficiency, generalization, and planning. Some works, such as COMBO \citep{zhang2024combo} and MABL \citep{venugopal2023mabl}, introduce compositional or hierarchical structures to separate shared dynamics from agent-specific behavior for better representation learning. DIMA \citep{loo2025efficient} and DCT \citep{zhang2024decentralized} adopt generative modeling frameworks, including diffusion processes and transformer architectures, to capture the complex, high-dimensional interactions among agents. GAWM \citep{shi2025gawm} proposes centralized attention modules to address observation inconsistencies in partially observable environments, while GRAD \citep{liu2024grounded} incorporates symbolic reasoning layers to enhance decision interpretability. DCWMC \citep{nomura2025decentralized} infers multi-agent states with a shared symbolic communication system. However, these approaches often rely on centralized supervision, task-specific structures, or passive inference, and lack generalization and scalability, thus, they cannot enable joint understanding, prediction, and
planning as a human-like intelligence in scalable and heterogeneous multi-agent systems. 

\vspace{-0.2cm}
\subsection{Theory of Mind}
ToM \citep{baker2011bayesian,rabinowitz2018machine,jara2019theory} endows agents with the ability to infer others' latent mental states, including goals and beliefs. Early works \citep{oguntola2023theory,zhang2025autotom} framed ToM as supervised behavior prediction or inverse planning over latent rewards and policies. Building on this foundation, ToM approaches have been applied in inverse reinforcement learning (IRL) \citet{ng2000algorithms,wu2023multiagent,gelpi2025towards}, intent-aware multi-agent reinforcement learning (MARL) \citep{qi2018intent,wang2022tomc}, and human reward learning \citep{tian2021learning}, to enhance representation learning and improve coordination by integrating inferred goals or agent-level cognition. To increase generality and reduce reliance on prior knowledge, AutoToM \citep{zhang2025autotom} and MUMA-ToM \citep{shi2025muma} adopt self-supervised and multimodal inference strategies with robustness and adaptability across agents. Hypothetical minds \citep{cross2024hypothetical} and LLM-ToM \citep{li2023theory} use large language models for explicit hypothesis generation and mental state attribution in complex social reasoning. However, most of these works require supervised or centralized multi-agent training and are sensitive to strategic behavior under partial observability \citep{rabinowitz2018machine}. They also need priors in the multi-agent structure, e.g., the number, roles, and topology of agents, thus, they cannot zero-shot generalize to new partners, group sizes, or interaction graphs \citep{zhang2025overcoming}. Moreovber, to the best of our knowledge, explicit ToM-style inference has not yet been used in world models for model-based planning.

\vspace{-0.3cm}
\section{Preliminaries}
\vspace{-0.1cm}
\subsection{Temporal Difference-Based World Model} 
\vspace{-0.1cm}
While the proposed Meta-ToM framework can apply to a broad range of world models, we use single-agent TD-MPC2, i.e., a temporal-difference (TD)-based learning approach with model predictive control (MPC), as a concrete framework to construct MetaMind for multi-agent systems. Given observation $\boldsymbol{o}_{\mathrm{t}}$ of the agent in single-agent settings, the TD-based world model \citep{hansen2023td} is given by
\vspace{-0.2cm}
\begin{equation}
\begin{aligned}
&\text{Belief Encoder: } \quad \boldsymbol{b}_t=\mathcal{B}_\varphi(\boldsymbol{o}_t,\boldsymbol{g}) \\
&\text{Forward Dynamics: } \quad \tilde{ \boldsymbol{b}}_{t+1}=\mathcal{D}_\varphi(\boldsymbol{b}_t,\boldsymbol{a}_t,\boldsymbol{g})\\
&\text{Reward Predictor: }  \quad \tilde r_t=\mathcal{R}_\varphi(\boldsymbol{b}_t,\boldsymbol{a}_t,\boldsymbol{g}) \\
&\text{Terminal Value: }    \quad \tilde q_t=\mathcal{Q}_\varphi(\boldsymbol{b}_t,\boldsymbol{a}_t,\boldsymbol{g})
\end{aligned}
\end{equation}
where $\boldsymbol{g}$ is a learnable goal vector, $\boldsymbol{b}_t$ is goal-oriented belief by $\mathcal{B}_{\varphi}$, $\boldsymbol{a}_t$ is an action from MPC, $\tilde{\boldsymbol{b}}_t$ is the predicted belief by $\mathcal{D}_{\varphi}$, $\tilde{r}_t$ is the predicted reward by $\mathcal{R}_{\varphi}$ at step $t$, and $\mathcal{Q}_{\varphi}$ estimates the discounted return with the parameter $\varphi$. We define a multi-step prediction loss:
\vspace{-0.3cm}
\begin{align}
\label{eq:main_loss}
\mathcal{L}_{0}(\varphi) & \!=\! \mathbb{E}_{\tau } \!\Big[ \sum_{t=0}^T \varsigma^t \big(\underbrace{\left\|\mathcal{D}_\varphi\left(\boldsymbol{b}_t, \boldsymbol{a}_t,\boldsymbol{g}\right) \!-\! \operatorname{sg}(\mathcal{B}_\varphi(\boldsymbol{o}_{t+1},\boldsymbol{g})) \right\|^2}_{\text{Imagination Loss}} \nonumber \\
&\qquad \qquad + \underbrace{\Im(\tilde{r}_t, r_t) + \Im(\tilde{q}_t, q_t)}_{\text{Reward and Value Loss}}\big)\Big],
\end{align}
where $\varsigma \in (0,1)$ is the time weighting,  $\tau = \{\boldsymbol{o}_t, \boldsymbol{a}_t, r_t, \boldsymbol{o}_{t+1}\}$ is a sampled actual trajectory, $\operatorname{sg}(\cdot)$ is the stop-gradient operation, and $q_t = r_t + \gamma Q_{\varpi}(\tilde{\boldsymbol{b}}_{t+1},  \pi_\varphi(\tilde{\boldsymbol{b}}_{t+1},\boldsymbol{g}))$ is the TD target at step $t$ with a factor $\gamma \in (0,1]$ and multi-Q exponential moving average (EMA) with parameters $\varpi$. To ensure scale-invariant learning and stable multi-goal learning, TD-MPC2 formulates the reward and value prediction as discrete regression in log-transformed space, where  $\Im(\cdot)$ is soft cross-entropy (More details can be seen in \underline{Appendix \ref{ap:multiq-ema} and \ref{ap:log}}).

\vspace{-0.2cm}
\subsection{Model Predictive Control} 
\vspace{-0.1cm}
We perform MPC in the differentiable belief space induced by $\mathcal{D}_{\varphi}$ as the planner of the world model. At each step, we optimize a Gaussian proposal $\mathcal{N}(\mu, \sigma^2)$ over candidate action sequences $\boldsymbol{a}_{t:t+\bar{H}}=\{\boldsymbol{a}_t, \boldsymbol{a}_{t+1},\cdots,\boldsymbol{a}_{t+\bar{H}}\}$ to maximize the expected cumulative return over horizon $\bar{H}$:
\vspace{-0.2cm}
\begin{align}
\left\{\mu^*, \sigma^* \right\} = \arg \max_{\mu, \sigma}  & \mathbb{E}_{\boldsymbol{a}_{t:t+\bar{H}}}  \left[\sum_{h=0}^{\bar{H}-1}  \gamma^h \mathcal{R}\left(\tilde{\boldsymbol{b}}_{t+h}, \boldsymbol{a}_{t+h},\boldsymbol{g}\right) \right. \nonumber \\
& \left.+\gamma^{\bar{H}} \mathcal{Q}\left(\tilde{\boldsymbol{b}}_{t+\bar{H}},\boldsymbol{a}_{t+\bar{H}},\boldsymbol{g}\right)\right].
\end{align}
After iterative optimization, the first action $\boldsymbol{a}_t$ is executed in the environment, while subsequent steps are warm-started by using $\pi_\varphi(\boldsymbol{b}_t, \boldsymbol{g})$ as initial mean $\mu_0$. 

\vspace{-0.3cm}
\section{Proposed MetaMind with Meta-ToM}
\vspace{-0.1cm}
As shown in Figure \ref{fig:sys}, MetaMind comprises two components: a goal-conditioned world model component for forward prediction and planning, and a Meta-ToM component for inverse inference. In this section, we will present the design of a novel, self-supervised Meta-ToM component for world models in multi-agent systems.

\paragraph{Inverse Inference.}
In multi-agent systems, an agent cannot rely solely on its own observations, as they are partial and their evolution is causally coupled through multi-agent interactions, while other agents' goals and beliefs remain hidden. To understand and learn the psychological principles with causal structures underlying behavior, as shown in Figure \ref{fig:sys}(a), we can perform inverse inference of latent \textit{mental states}, including both beliefs and goals, from externally observable behavioral trajectories. Let $\Psi_\theta$ be the inverse goal predictor and $\Omega_\theta$ be the inverse belief predictor with parameters $\theta$. Then, the estimated goal and belief can be, respectively, obtained by $\hat{\boldsymbol{g}}_{t}=\Psi_\theta(\boldsymbol{a}_{<t},\hat{\boldsymbol{g}}_{t-1})$ and $\hat{\boldsymbol{b}}_t= \Omega_\theta(\boldsymbol{a}_t,\hat{\boldsymbol{g}}_k)$. 
We optimize $\Psi_\theta$ and $\Omega_\theta$ by minimizing the loss of consistency of inverse reasoning:
\vspace{-0.2cm}
\begin{equation}
\begin{aligned}
\label{eq:inverse}
\mathcal{L}_{\text{inv}}(\theta) & = \mathbb{E}_{\tau } \Big[ \sum_{k=1}^T \| \Psi_\theta(\boldsymbol{a}_{<k},\hat{\boldsymbol{g}}_{k-1}) - \boldsymbol{g}  \|^2 \\
& \qquad \qquad \quad + \|\Omega_\theta(\boldsymbol{a}_{k},\mathrm{sg}(\hat{\boldsymbol{g}}_k))-\boldsymbol{b}_t\|^2\Big],
\end{aligned}
\end{equation}
where $\mathrm{sg}(\cdot)$ is the stop-gradient operator.
As illustrated in Figure \ref{fig:sys}(b), the inverse inference models $\Psi_\theta$ and $\Omega_\theta$ ground Meta-ToM by understanding potential mental states from its own behavioral trajectories, i.e., thinking ``\textit{Why and when I did these actions?}'' in a self-supervised manner. 

\paragraph{Self-Reflection.}
Since inverse inference from behavior is generally non-identifiable, i.e., multiple latent explanations for the same behavior, which can lead to ambiguous mental states. Hence, we introduce an explicit self-reflection by cycle-consistency to make the inverse inference behaviorally identifiable. Particularly, the agent verifies whether its predicted beliefs and goals can regenerate original behavioral trajectories, i.e., ``\textit{whether I will take the same action at the inferred belief and goal.}'', as illustrated in Figure \ref{fig:sys}(b). Given a behavioral trajectory $\boldsymbol{a}_{0:T}$, the latent mental states $\hat{\boldsymbol{b}}_t, \hat{\boldsymbol{g}}_t$ from the inverse models are passed to the forward policy model $\pi_\varphi$ and reconstruct actions with $\hat{\boldsymbol{a}}_t = \pi_\varphi(\hat{\boldsymbol{b}}_t, \hat{\boldsymbol{g}}_t)$. Hence, the belief-action cycle-consistency loss by minimizing the discrepancy between the predicted actions $\hat{\boldsymbol{a}}_t$ and original actions $\boldsymbol{a}_{0:T}$ is given by:
\vspace{-0.2cm}
\begin{equation}
\label{eq:loop}
\mathcal{L}_{\text{ref}}(\theta) =  \mathbb{E}_{\tau } \Big[\sum_{k=1}^T \left\| \pi_\varphi(\hat{\boldsymbol{b}}_{T-k}, \hat{\boldsymbol{g}}_{T-k}) - \boldsymbol{a}_{T-k} \right\|^2 \Big].
\end{equation}
The cycle consistency objective (\ref{eq:loop}) ensures that inferred mental states are not only plausible but also behaviorally sufficient. This reflective loop plays a dual role that serves both as a self-improvement signal for $\Psi_\theta$ and $\Omega_\theta$, and as a diagnostic tool for identifying mismatches among beliefs, goals, and actions, thus stabilizing inverse reasoning.

\paragraph{Analogical Inference.}
Unlike the internal, inaccessible beliefs and goals of each agent, actions are universally recognized, externally observable and protocol-agnostic across heterogeneous agents \citep{fudenberg1998theory,baker2009action,baker2017rational}. Hence, the Meta-ToM framework endows the agent with an analogical inference ability, which zero-shot generalizes the cognitive ability from first-person reasoning to third-person inference, i.e., from thinking ``Why and when did I perform these actions'' to thinking ``Why and when they did these actions''. Particularly, the learned inverse models $(\Psi^{i}_{\theta}, \Omega^{i}_{\theta})$ of agent $i$ from self-behavioral trajectories can be used for mental state estimation of other agents. Given an observable action trajectory $\{\boldsymbol{a}_t^{(i,j)}\}$ of agent $j$, agent $i$ can actively estimate the goals and beliefs of agent $j$, respectively, by analogical inference:
\begin{equation}
\label{eq:inv}
\hat{\boldsymbol{g}}_{t}^{(i,j)}:= \Psi^{i}_{\theta} \left(\boldsymbol{a}_{t}^{(i,j)}, \hat{\boldsymbol{g}}_{t-1}^{(i,j)}\right), \hat{\boldsymbol{b}}_t^{(i,j)}:=  \Omega^{i}_{\theta} \left(\boldsymbol{a}_t^{(i,j)}, \hat{\boldsymbol{g}}_t^{(i,j)}\right).
\end{equation}
In this way, the Meta-ToM enables explicit multi-agent modeling.
Based on estimated $\hat{\boldsymbol{g}}_{t}^{(i,j)}$ and $\hat{\boldsymbol{b}}_t^{(i,j)}$ from Meta-ToM, agent $i$ predicts beliefs and actions of agent $j$ with the forward models, respectively, by $\tilde{\boldsymbol{b}}^{(i,j)}_{t+1} = \mathcal{D}_{\varphi}(\hat{\boldsymbol{b}}_t^{(i,j)}, \boldsymbol{a}_t^{(i,j)}, \hat{\boldsymbol{g}}_{t}^{(i,j)})$ and $\tilde{\boldsymbol{a}}^{(i,j)}_{t+1} = \pi_{\varphi}(\hat{\boldsymbol{b}}^{(i,j)}_{t}, \hat{\boldsymbol{g}}_{t}^{(i,j)})$ for reliable imagination rollouts without centralized information in multi-agent systems. 

\vspace{-0.2cm}
\paragraph{Self-Supervised Collective Beliefs.}
In multi-agent environments, an ego agent must plan under uncertainty about others' latent beliefs and goals. We therefore construct a \emph{collective belief} that aggregates inferred mental states in a permutation-invariant and goal-conditioned manner, producing a compact interaction-aware latent representation. This representation summarizes \emph{who is around, what they likely intend, and how they may respond}, thus enabling scalable long-horizon planning.
Let $\mathcal{E}_{\varrho}^i(\cdot)$ be a belief fusion module of agent $i$ for collective beliefs with parameters $\varrho$. Given the actual belief $\boldsymbol{b}^i_{t+1}$ of agent $i$, we optimize $\mathcal{E}_{\varrho}(\cdot)$ by a self-supervised objective:
\vspace{-0.2cm}
\begin{equation}
\label{eq:selfsup_fusion}
\mathcal L_{\text{col}}(\varrho)
= \big\|\mathcal{E}_{\varrho}^i(\{\tilde{\boldsymbol{b}}_{t+1}^{(i,j)},\hat{\boldsymbol{g}}_{t}^{(i,j)}\})-\operatorname{sg}(\boldsymbol{b}^i_{t+1})\big\|_2^2.
\end{equation}
The loss function in (\ref{eq:selfsup_fusion}) provides a more accurate imagination of multi-agent interactions and enables emergent group awareness, thus enabling long-term planning without explicit communication or prior information. Moreover, any permutation-invariant module for belief fusion. Here, we apply the transformer-based aggregation approach \cite{zhang2024decentralized}, as shown in \underline{Appendix \ref{ap:agg}}. The training details of the proposed MetaMind with a Meta-ToM framework is summarized in \underline{Appendix \ref{ap:alg}}.

\section{Goal Identification by Meta-ToM}
In this section, we prove \emph{goal identification}, i.e., distinct goals induce distinguishable behavioral trajectories, for the effectiveness of the proposed Meta-ToM. In particular, goal identification guarantees that minimizing the behavioral residual leads to accurate estimations of goals.  Moreover, the representation-equivalence class is the invariant shared across agents. While the set of internal representations may vary, the set of possible tasks $\mathcal{T}$ is fixed. Formally, let $\Gamma$ be the group of representation transformations mapping heterogeneous encodings $\boldsymbol{g}^j$ to the same task $\kappa\in\mathcal{T}$. Goal identification is therefore proven at the task-equivalence level: for any two distinct tasks $\kappa_1$, $\kappa_2$ with encodings $\boldsymbol{g}_1$, $\boldsymbol{g}_2$, respectively, the induced behavioral distributions must be distinguishable, i.e.,
$
    p(\boldsymbol{a}_{t:t+H}\mid \boldsymbol{g}_1)\;\neq\; p(\boldsymbol{a}_{t:t+H}\mid \boldsymbol{g}_2),
$
while encodings within the same equivalence class $[\boldsymbol{g}]$ are treated as indistinguishable, and identification guarantees task-level intentions rather than agent-specific internal representations.
We omit the subscript parameters $\theta$, $\varphi$, and the agent index $i$ for brevity in the next discussions. 
\begin{definition}[Behavioral TD residual]
\label{def1}
Given an observation horizon $H \ge 2$ and a behavioral trajectory
$\boldsymbol{a}_{0:H-1}=(\boldsymbol{a}_0,\ldots,\boldsymbol{a}_{H-1})$.
For any candidate goal $\boldsymbol{g} \in\mathcal G$, we have $\boldsymbol{b}_t(\boldsymbol{g}):=\Omega(\boldsymbol{a}_t,\boldsymbol{g})$, $
\tilde{\boldsymbol{b}}_{t+1}(\boldsymbol{g}):=\mathcal{D}\big(\boldsymbol{b}_t(\boldsymbol{g}),\boldsymbol{a}_t,\boldsymbol{g}\big)$, and $
\tilde{\boldsymbol{a}}_{t+1}(\boldsymbol{g}):=\pi\big(\tilde{\boldsymbol{b}}_{t+1}(\boldsymbol{g}),\boldsymbol{g}\big)$.
Equivalently, let $F_{\boldsymbol{g}}(\boldsymbol{a}):=\pi\!\big(\mathcal{D}(\Omega(\boldsymbol{a},\boldsymbol{g}),\boldsymbol{a},\boldsymbol{g}),\,\boldsymbol{g}\big)$, and then $\tilde{\boldsymbol{a}}_{t+1}(\boldsymbol{g})=F_{\boldsymbol{g}}(\boldsymbol{a}_t)$.
The one-step TD residual and the stacked long-horizon residual are, respectively, defined by
\begin{equation}
  \begin{aligned}
&\phi_t(\boldsymbol{g}):=\boldsymbol{a}_{t+1} - \tilde{\boldsymbol{a}}_{t+1} (\boldsymbol{g})=\boldsymbol{a}_{t+1}-F_{\boldsymbol{g}}(\boldsymbol{a}_t), \\
&\Phi_H(\boldsymbol{g}):=\big[\phi_0(\boldsymbol{g});\ \phi_1(\boldsymbol{g});\ \ldots;\ \phi_{H-2}(\boldsymbol{g})\big].
\end{aligned}
\end{equation} 
We use $d_H(\boldsymbol{g}):=\|\Phi_H(\boldsymbol{g})\|_2$ 
as the behavioral cognition loss for goal identification. Since $\mathcal{B}$, $\Omega$ and $\pi$ are continuously differentiable, $\Phi_H(\boldsymbol{g})$ is differentiable in goal $\boldsymbol{g}$.
\end{definition}

\begin{definition}[Residual margin for discrete goal identification]
\label{def3}
Let $\mathcal C=\{\boldsymbol{g}^1,\ldots,\boldsymbol{g}^K\}\subset\mathcal G$ be a discrete codebook that records learned goal representations $\{\boldsymbol{g}^k\}$, and
let the true goal be $\boldsymbol{g}^*=\boldsymbol{g}^{k^*} \in \mathcal{C}$.
Given $\Phi_H(\boldsymbol{g})$ and $d_H(\boldsymbol{g}):=\|\Phi_H(\boldsymbol{g})\|_2$, the residual margin of $\mathcal{C}$ at horizon $H$ is
\begin{equation}
\bar\gamma_H(\mathcal C;\boldsymbol{g}^*)\ :=\ \min_{k\neq k^*}\ \big(\|\Phi_H(\boldsymbol{g}^k)-\Phi_H(\boldsymbol{g}^*)\|_2\big).
\end{equation}
\end{definition}

\begin{theorem}[Goal identification]
\label{thm:M2}
Let observable behavior trajectory $\boldsymbol{a}^j_{0:H-1}$ be generated by $\pi^j$ and $\mathcal{D}^j$ under the true goal $\boldsymbol{g}^*$. Assume $\pi(\boldsymbol{b},\boldsymbol{g})$ is $\xi_\pi$-Lipschitz in $\boldsymbol{b}$ and $\mathcal{D}(\cdot,\boldsymbol{a}^j,\boldsymbol{g})$ is $\xi_{\mathcal{D}}$-Lipschitz in $\boldsymbol{b}$. The per-step mismatches are given by
$\Delta \boldsymbol{b}^j_t:=\Omega(\boldsymbol{a}_t,\boldsymbol{g}^*)-\boldsymbol{b}_t^j$, $\Delta \mathcal{D}_t(\boldsymbol{b}):=\mathcal{D}(\boldsymbol{b},\boldsymbol{a}_t,\boldsymbol{g}^*)-\mathcal{D}^j(\boldsymbol{b},\boldsymbol{a}_t,\boldsymbol{g}^*)$, and $
\Delta\pi_{t+1}:=\pi^j(\boldsymbol{b}_{t+1}^j,\boldsymbol{g}^*)-\pi(\boldsymbol{b}_{t+1}^j,\boldsymbol{g}^*)$.
The behavioral-mismatch energy is:
\begin{equation}
\delta_{\mathrm{act}}
:=\left(\sum_{t=0}^{H-2}\Big(\,\|\Delta\pi_{t+1}\|
+\xi_\pi \|\varepsilon_t\|\Big)^2\right)^{1/2},
\end{equation}
where $\varepsilon_t:= \|\Delta \mathcal{D}_t(\boldsymbol{b}_t^j)\|+\xi_{\mathcal{D}}\|\Delta \boldsymbol{b}^j_t\|$.
If $\delta_{\mathrm{act}}<\bar\gamma_H(\mathcal C;\boldsymbol{g}^*)/2$, then the nearest-neighbor $\hat{k}=\arg\min_{k}\|\Phi_H(\boldsymbol{g}^k)\|_2$ returns
$\hat{\boldsymbol{g}}=\boldsymbol{g}^{\hat{k}}=\boldsymbol{g}^*$.
\end{theorem}
\begin{proof}
    Please see \underline{Appendix \ref{proof:a2}}.
\end{proof}

\begin{lemma}[Local linear margin via Lipschitz Jacobian]
\label{lemma:local-margin}
Assume $J_\Phi$ is $L_J$-Lipschitz in $\boldsymbol{g}$ on a neighborhood of $\boldsymbol{g}^*$ and every $\boldsymbol{g}^k$ lies in the locality ball
$\|\boldsymbol{g}^k-\boldsymbol{g}^{k^*}\|\le \phi_\star$ with $\phi_\star:=\sigma_{\min}\!\big(J_\Phi(\boldsymbol{g}^*)\big)/L_J$ and $\rho_{\min}:=\min_{k\neq \ell}\|\boldsymbol{g}^k-\boldsymbol{g}^\ell\|$, where $\sigma_{\min}$ is the smallest singular value. Then, for each $k\neq k^*$,
\begin{equation}
  \begin{aligned}
&\|\Phi_H(\boldsymbol{g}^k)\|_2\ \ge\ \frac12\,\sigma_{\min}\!\big(J_\Phi(\boldsymbol{g}^*)\big)\,\|\boldsymbol{g}^k-\boldsymbol{g}^{k^*}\|,
\\
& \quad \Rightarrow
\gamma_H(\mathcal C;\boldsymbol{g}^*)\ \ge\ \frac12\,\sigma_{\min}\!\big(J_\Phi(\boldsymbol{g}^*)\big)\ \rho_{\min}.
\end{aligned}
\end{equation}
\end{lemma}
\begin{proof}
    Please see \underline{Appendix \ref{proof:l1}}.
\end{proof}

\begin{figure*}[ht!]
    \centering
    \includegraphics[width=1\linewidth]{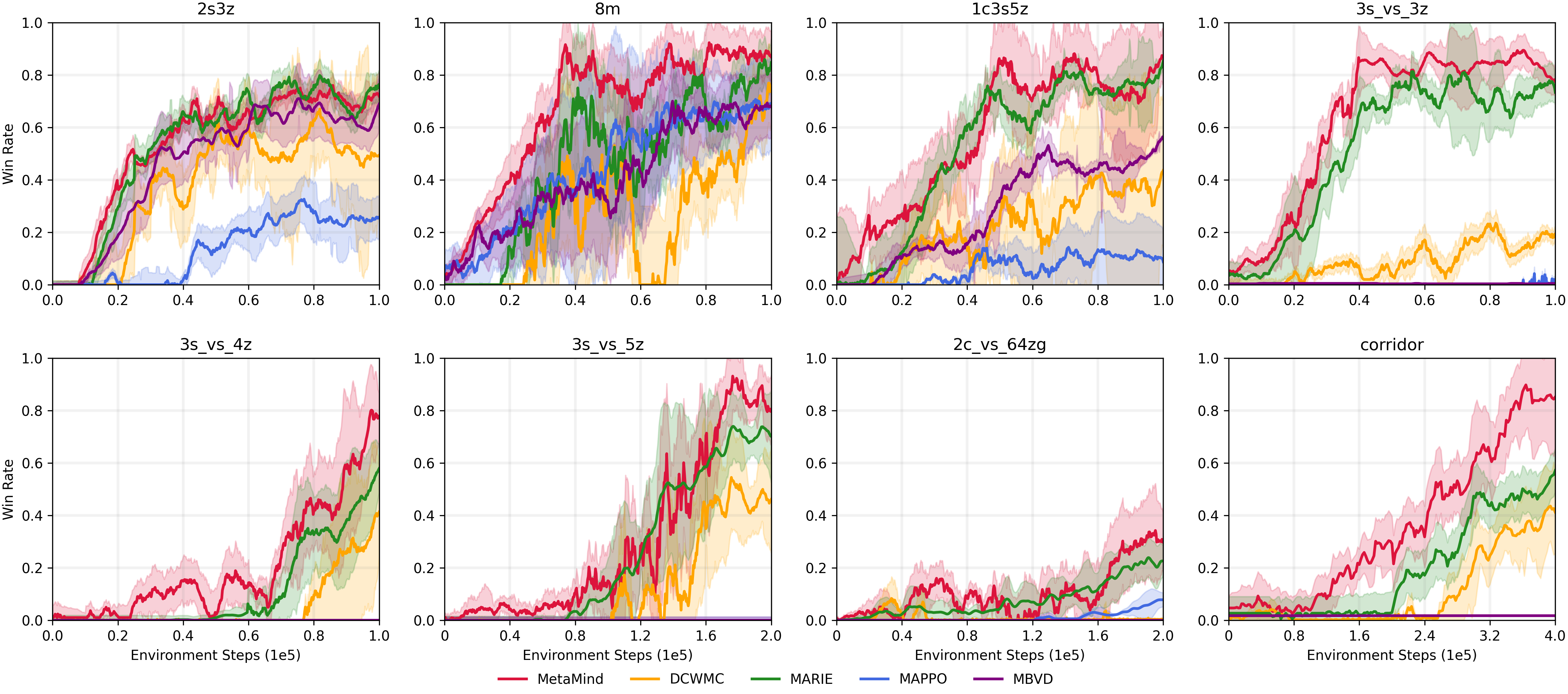}
    \vspace{-0.5cm}
    \caption{Performance comparison on 8 SMAC tasks under limited environment steps. The win rates are averaged over 100 test episodes.}
    \vspace{-0.5cm}
    \label{fig:es8}
\end{figure*}

\begin{theorem}[Sufficient horizon size for goal identification]
\label{thm:a2}
If there exist $m$ and $\beta>0$ such that for every window $[k,k+m-1]\subseteq[0,H-2]$, discrete-time sliding-window persistence of excitation is
\begin{equation}
\label{eq:PE}
\sum_{t=k}^{k+m-1}\Big(\frac{\partial \phi_t}{\partial \boldsymbol{g}}(\boldsymbol{g}^*)\Big)^{\!\top}
\Big(\frac{\partial \phi_t}{\partial \boldsymbol{g}}(\boldsymbol{g}^*)\Big)\ \succeq\ \beta\, I_{d_{\boldsymbol{g}}}.
\end{equation}
Then, with codebook separation $\rho_{\min}>0$, the sufficient horizon size $H$ for robust goal identification is:
\begin{equation}
H\ \ge\ 1\;+\;m\Big(\big\lfloor L\big\rfloor+1\Big),\qquad
L\ :=\ \Big(\frac{4\,\delta_{\mathrm{act}}}{\rho_{\min}\sqrt{\beta}}\Big)^{\!2}.
\end{equation}
\end{theorem}
\begin{proof}
    Please see \underline{Appendix \ref{proof:a3}}.
\end{proof}
\begin{remark}
    Theorems \ref{thm:M2}-\ref{thm:a2} and Lemma~\ref{lemma:local-margin} are applicable to both homogeneous agents with the same network parameters, and heterogeneous agents with different network parameters or structures.
\end{remark}
Theorem \ref{thm:M2} converts Meta-ToM from a heuristic to a principled method by proving goal identification. In particular, Theorem \ref{thm:M2} introduces a behavioral-mismatch energy $\delta_{\mathrm{act}}$ and proves correctness whenever a simple criterion holds, $\bar\gamma_H>2\,\delta_{\mathrm{act}}$. It justifies the self-reflection and cycle-consistency objective since they reduce $d_H(g^*)$ and enlarge the effective margin. Lemma~\ref{lemma:local-margin} links smoothness and codebook separation to finite-sample separability, and reveals that increasing $\sigma_{\min}(J_\Phi(\boldsymbol g^*))$ or $\rho_{\min}$ enlarges the margin and improves robustness to estimation errors. A sufficient horizon size for goal identification is proved in Theorem \ref{thm:a2}, which reveals larger behavioral mismatch $\delta_{\mathrm{act}}$ requires longer observation, and more informative trajectories by larger $\beta$ reduce the required horizon.
Theorem~\ref{thm:a2} shows \emph{when inverse inference should succeed with finite trajectories} and provides guidance for choosing the trajectory chunk length used by Meta-ToM in practice. In a nutshell, Theorems \ref{thm:M2}-\ref{thm:a2} and Lemma~\ref{lemma:local-margin} explain why inverse inference by Meta-ToM is effective, when zero-shot analogical inference will succeed, and how to ensure closed-loop reliability in multi-agent settings.

\section{Experiments}
\subsection{Experimental Settings}
We evaluate all methods across StarCraft Multi-Agent Challenge (SMAC) tasks, which requires coordination under partial observability with decentralized execution.
Following common practice, we use a diverse set of maps covering both homogeneous and heterogeneous team compositions. Unless otherwise specified, an episode terminates when the combat ends or when the environment reaches its default time limit.
We compare MetaMind against representative multi-agent world model methods, and CTDE multi-agent reinforcement learning baselines. In particular, the baselines include MARIE~\cite{zhang2024decentralized} as CTDE world models, DCWMC~\cite{nomura2025decentralized} as decentralized world models, and MAPPO~\cite{yu2022surprising} and MBVD~\cite{xu2022mingling} as CTDE reinforcement learning (RL) baselines.
All methods are trained with the same environment interaction budget and evaluated under the same protocol. Experiments are run on an NVIDIA GeForce RTX~4090 GPU with 24~GB~VRAM.
Detailed model architectures, hyperparameters, and task-specific settings are provided in \underline{Appendix \ref{hyp}}.

\begin{table*}[ht!]  
\centering
\caption{Performance comparison of win rate (Mean ± Std.) across 13 SMAC maps for different methods, averaged over 100 random seeds after 200k environment steps of training. The \texttt{corridor} task is trained for 400K steps, the \texttt{3s\_vs\_5z} and \texttt{2c\_vs\_64zg} tasks are trained for 200K steps, and all remaining tasks are trained for 100K steps.}
\label{tab:smac_main}
\renewcommand{\arraystretch}{1.15}
\setlength{\tabcolsep}{6pt}
\small 
\begin{tabular}{l c c c c c c c}
\toprule
Map &Type & \multicolumn{3}{c}{Multi-Agent World Model Methods} & \multicolumn{2}{c}{Multi-Agent RL Methods}\\
\cmidrule(lr){3-7}
& & 
\makecell{MetaMind\\(\textbf{Ours})} &
\makecell{DCWMC\\{\cite{nomura2025decentralized}}} &
\makecell{MARIE\\{\cite{zhang2024decentralized}}} &
\makecell{MAPPO\\{\cite{yu2022surprising}}} &
\makecell{MBVD\\{\cite{xu2022mingling}}} \\
\cmidrule(lr){3-7} 
 & & Decentralized & Decentralized & CTDE & CTDE & CTDE \\
\midrule

2m\_vs\_1z       & \multirow{10}{*}{Homog.} &
91.7 $\pm$ 6.2 & 75.1 $\pm$ 10.3 & \textbf{93.1 $\pm$ 3.3} & 84.6 $\pm$ 8.1  & 38.2 $\pm$ 19.1 \\

2s\_vs\_1sc      &   &
93.2 $\pm$ 5.4 & 80.4 $\pm$ 9.2 & 96.1 $\pm$ 7.5  & 93.1 $\pm$ 6.1  & 16.1 $\pm$ 14.2 \\

3m              &   &
95.5 $\pm$ 3.1 & 72.1 $\pm$ 10.7 & 96.2 $\pm$ 1.1 & 82.3 $\pm$ 11.9  & 72.1 $\pm$ 8.6 \\ 

8m              &   &
89.1 $\pm$ 7.7 & 70.5 $\pm$ 12.8 & 84.1 $\pm$ 7.2 & 69.4 $\pm$ 17.8  & 71.2 $\pm$ 11.2 \\ 

3s\_vs\_3z       &   &
95.1 $\pm$ 2.1 & 63.4 $\pm$ 9.0 & 97.3 $\pm$ 1.9 & 6.8 $\pm$ 3.2  & 1.0 $\pm$ 0.5 \\ 

3s\_vs\_4z       &  &
77.3 $\pm$ 9.3 & 41.3 $\pm$ 17.2 & 71.2 $\pm$ 8.5 & 0.0 $\pm$ 0.0  & 0.0 $\pm$ 0.0 \\

3s\_vs\_5z       &   &
80.4 $\pm$ 7.3 & 46.1 $\pm$ 12.3 & 72.4 $\pm$ 13.2 & 0.0 $\pm$ 0.0  & 0.0 $\pm$ 0.0 \\

2c\_vs\_64zg     &    &
29.8 $\pm$ 11.2 & 9.5 $\pm$ 4.5 & 23.5 $\pm$ 15.1 & 5.2 $\pm$ 10.3  & 0.0 $\pm$ 0.0 \\

baneling   & &
92.7 $\pm$ 5.1 & 54.7 $\pm$ 18.9 & 91.8 $\pm$ 6.8 & 31.3 $\pm$ 19.4  & 11.3 $\pm$ 13.2 \\

corridor          &   &
76.8 $\pm$ 14.2 & 41.2 $\pm$ 15.7 & 72.5 $\pm$ 16.1 & 0.0 $\pm$ 0.0 & 0.0 $\pm$ 0.0 \\  

\midrule

2s3z            & \multirow{3}{*}{Heterog.} &
84.6 $\pm$ 6.7 & 54.8 $\pm$ 11.3 & 82.3 $\pm$ 8.5 & 31.2 $\pm$ 14.2  & 42.7 $\pm$ 8.1 \\ 

1c3s5z           &  &
87.6 $\pm$ 8.3 & 58.6 $\pm$ 29.2 & 85.0 $\pm$ 9.4 & 18.4 $\pm$ 11.0   & 60.9 $\pm$ 11.4 \\

MMM              & &
85.7 $\pm$ 3.6 & 68.1 $\pm$ 10.3 & 83.2 $\pm$ 7.1 & 11.2 $\pm$ 3.5  & 10.5 $\pm$ 7.1 \\
\bottomrule
\end{tabular}
\end{table*}

\begin{figure*}[ht!]
    \centering
    \vspace{-0.1cm}
    \includegraphics[width=1\linewidth]{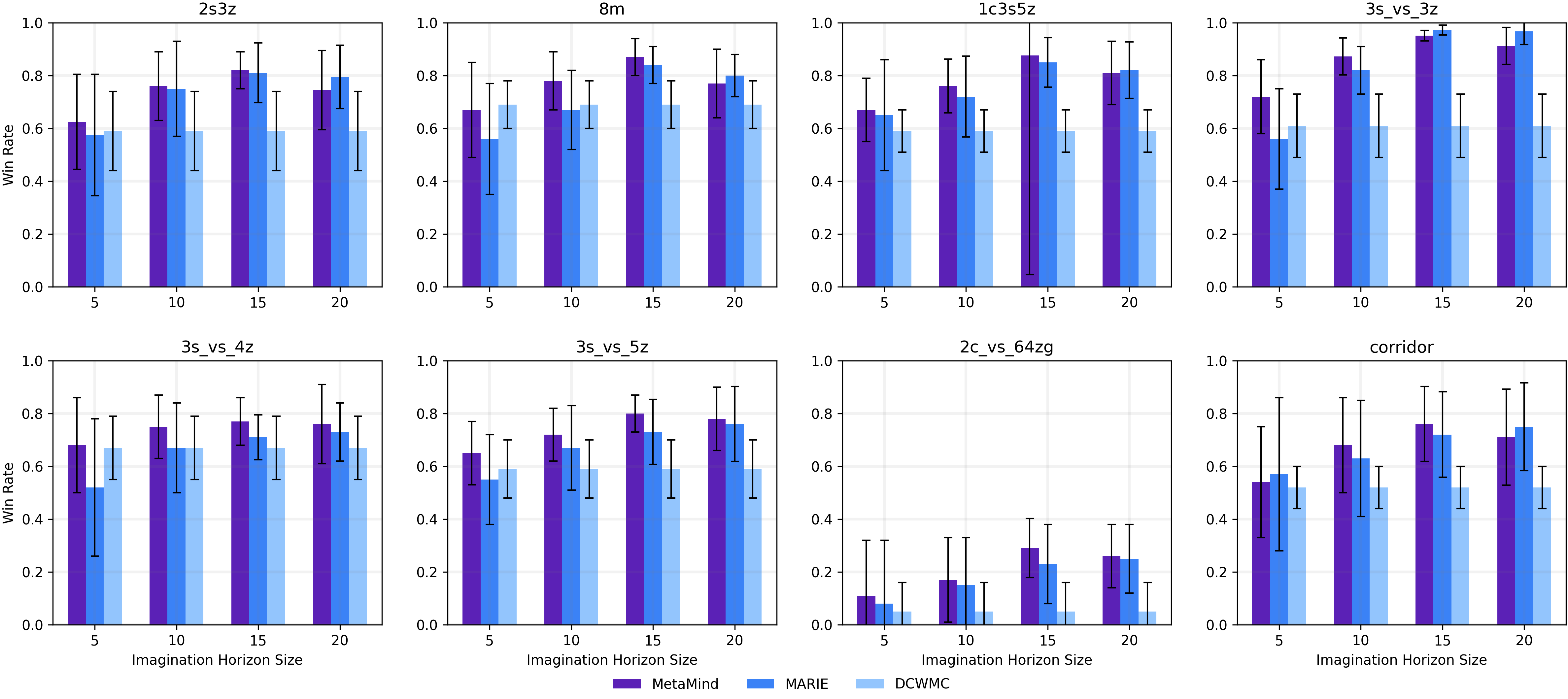}
    \vspace{-0.5cm}
    \caption{Performance comparison of win rate (Mean ± Std.) on 8 SMAC tasks versus different imagination horizon length.}
    \vspace{-0.4cm}
    \label{fig:H8}
\end{figure*}

\vspace{-0.1cm}
\subsection{Task Performance}
Figure~\ref{fig:es8} shows the convergence performance of different methods across 8 SMAC tasks versus limited environment steps, and Table \ref{tab:smac_main} lists the task performance across 13 SMAC tasks within limited environment steps. We observe that MetaMind achieves better convergence and win rates within limited environment steps. As shown in Figure~\ref{fig:es8}, MetaMind reaches 80\% win rate within 200K steps on {3s\_vs\_5z}, whereas MARIE and DCWMC can only reach 75\% and 52\% win rate, respectively. This is because that MetaMind enables explicit, decentralized teammate mental-state inference for a multi-agent model-based planner, which reduces the effective non-stationarity from reactive agents, thereby improving sample efficiency and reliable rollouts for long-horizon planning. This advantage becomes most salient on heterogeneous, strategy-intensive maps where successful play requires complex anticipatory coordination rather than simple reactive control. 
From Table \ref{tab:smac_main}, in \texttt{3s\_vs\_4z}, \texttt{3s\_vs\_5z}, and \texttt{corridor} tasks, MAPPO and MBVD cannot achieve non-trivial coordination, whereas MetaMind achieves 77.3$\pm$9.3\%, 80.4$\pm$7.3\%, and 76.8$\pm$14.2\%, respectively. Compared to MARIE and DCWMC, MetaMind achieves up to 11.0\% improvement in terms of win rate on \texttt{3s\_vs\_5z}.
Overall, by explicitly inferring latent mental states and aggregating them into a permutation-invariant interaction-aware collective belief for model-based planning, MetaMind achieves faster convergence and superior win rates, particularly in heterogeneous tasks requiring robust long-horizon coordination under partial observability.

\begin{figure*}[ht]
    \centering
    \includegraphics[width=0.95\linewidth]{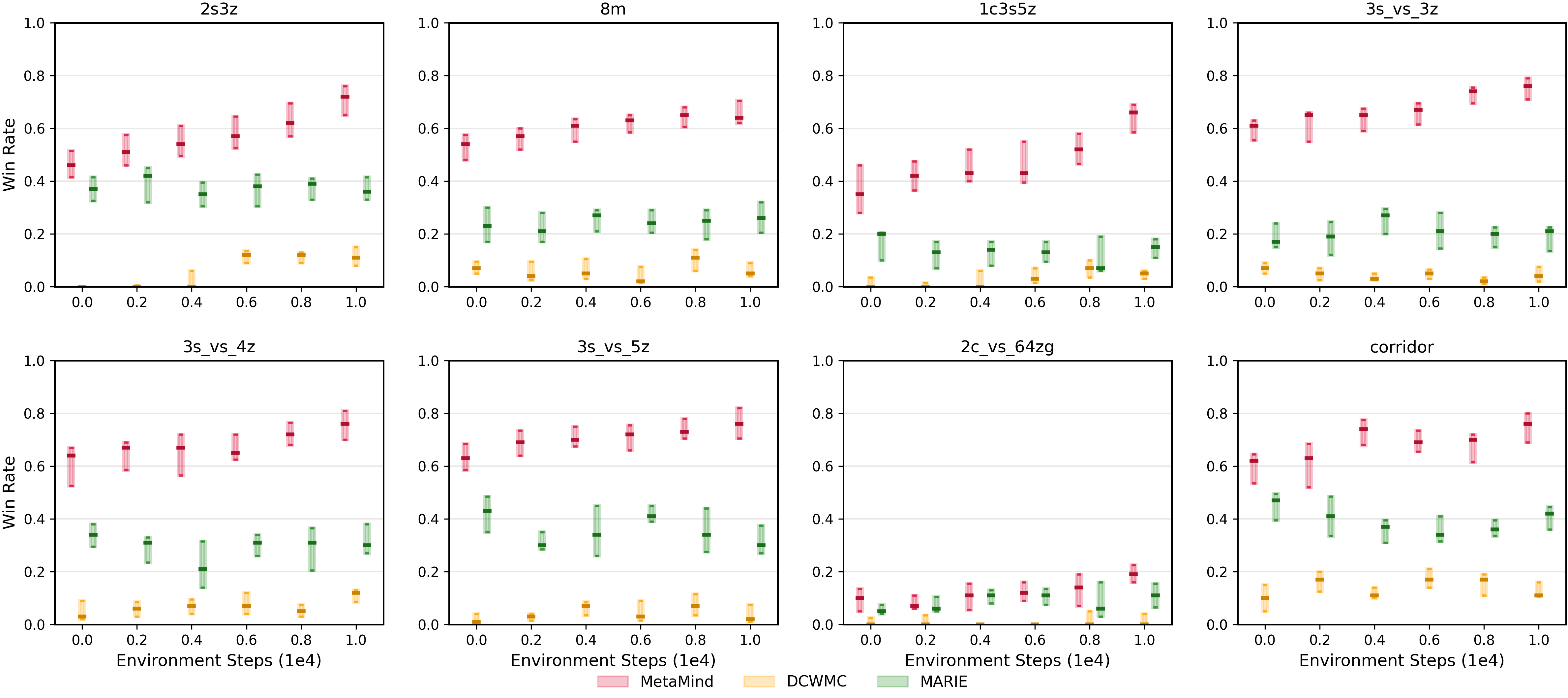}
    \vspace{-0.2cm}
    \caption{Performance comparison of win rate in few-shot multi-agent generalization on 8 SMAC tasks under limited environment steps.}
    \vspace{-0.3cm}
    \label{fig:G8}
\end{figure*}

\begin{figure}[ht]
    \centering
    \includegraphics[width=1\linewidth]{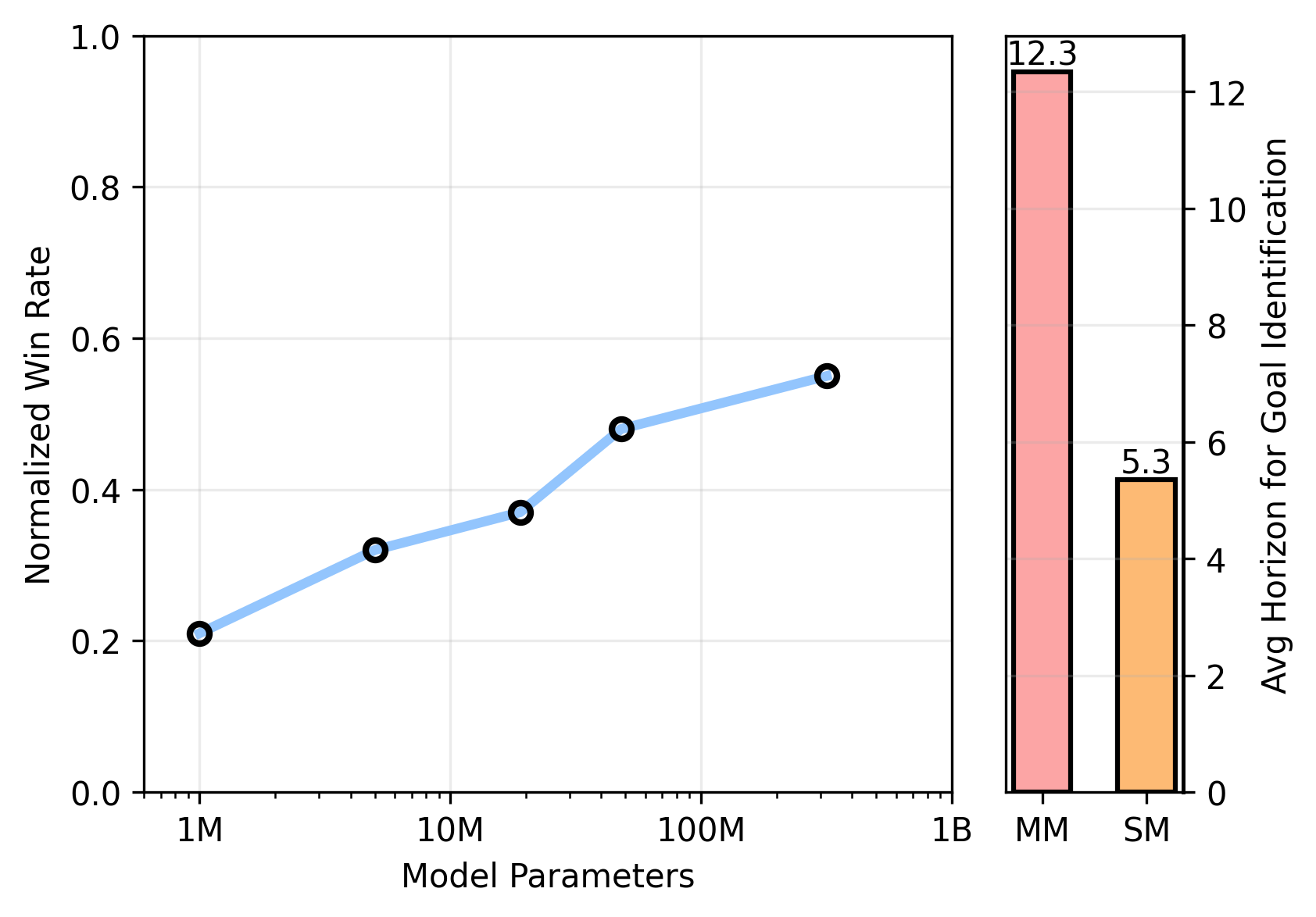}
    \vspace{-0.5cm}
    \caption{The scalability of a multi-map (MM) MetaMind to perform and cooperate over 13 SMAC maps, and the average horizon for goal identification of MT and single-map (SM) MetaMinds.}
    \vspace{-0.6cm}
    \label{fig:SCA}
\end{figure}

\subsection{Impact of Imagination Horizon}
\vspace{-0.15cm}
Figure \ref{fig:H8} studies how the imagination horizon affects multi-agent performance, with horizon size $H \in \{5,10,15,20\}$. DCWMC predicts only the next-step information and therefore is not affected by varying $H$. From Figure 3, a longer imagination horizon generally leads to higher win rates for MetaMind and MARIE, when $H$ increases from 5 to 15.
This improvement is significant for coordination-intensive tasks that require anticipating delayed tactical outcomes, where short-horizon rollouts tend to be short-sighted.
When $H$ reaches 20, the benefit becomes smaller, which is caused by compounding model errors and increased uncertainty in long-horizon imagined trajectories. MetaMind reaches the best win rates for $H \ge 10$ and shows a smaller performance reduction at $H{=}20$ than MARIE on long-horizon coordination tasks.
For example, on \texttt{2c\_vs\_64zg}, MetaMind achieves 43.0 \% at $H{=}15$, which outperforms MARIE by 4.6 \%, and achieves reliable performance when the horizon increases to $H{=}20$.
Figure \ref{fig:H8} suggests that the proposed meta-ToM and collective belief modeling can stabilize planning under partial observability when the imagination horizon becomes longer.
Based on these results, all remaining experiments use $H{=}15$ as the default horizon, which provides a balance between lookahead and rollout reliability.

\subsection{Multi-Agent Generalization}
In Figure \ref{fig:G8}, we evaluate multi-agent generalization with few environment interactions under teammate distribution shift. In particular, for each SMAC map, we first train a population of world models for agents. During testing, we form evaluation teams by randomly replacing one agent in a converged team with a world model drawn from this population, while keeping the remaining agents fixed. This creates a distribution shift in observable behaviors and coordination conventions without changing the environmental dynamics. This protocol aims to measure (a) teammate generalization, i.e., robustness to partner-policy changes; (b) coordination robustness, i.e., the ability to maintain effective collaboration when implicit coordination protocols differ; and (c) compositional generalization, i.e., whether agents trained under different runs can be composed into a functioning team. We propose a few-shot learning approach for MetaMind to fine-tune inverse inference, as given in \underline{Appendix \ref{sec:fewshot-inv-belief}}.

As shown in Figure \ref{fig:G8}, MetaMind exhibits stronger generalization under controlled partner-policy replacement. It achieves an average median win rate of 62.7\% across 8 SMAC maps, whereas MARIE reaches only 25.9\%. Moreover, MetaMind is robust in the lower tail, achieving a win rate of 59.4\% at the 25th percentile versus 23.0\% by MARIE. Compared to MARIE, MetaMind achieves an average 2.4-fold improvement in few-shot generalization. This indicates MetaMind has far fewer catastrophic coordination failures when implicit conventions shift. This is because: (a) analogical inference maps observed actions to latent goal/belief, thereby reducing reliance on a single co-adapted protocol, such as the emergent communication systems in DCWMC; (b) self-reflection regularizes inverse inference toward behaviorally identifiable explanations, thereby improving transfer to unseen partner policies; and (c) collective-belief aggregation absorbs local partner changes into an interaction-aware belief. Moreover, it is observed that the replacement of teammate breaks the sender-receiver codebook alignment of DCWMC, so messages become uninterpretable, thereby causing the communication channel failures and collapsed coordination performance.

\subsection{Scalability}
Figure~\ref{fig:SCA} shows that MetaMind scales with both model capacity and task diversity. As the parameter budget increases from 1M to 200M, the normalized win rate improves from 0.21 to 0.55. It indicates that larger models effectively translate additional capacity into better multi-task coordination rather than being bottlenecked by negative transfer. Moreover, the average inference cost of goal identification for MM MetaMind is 12.3, compared to 5.3 for SM MetaMind. This indicates that even under multi-map training, MetaMind can still accurately infer others' goals/beliefs and remain reliable across diverse maps with coordination. Such robust intent identification provides a stable basis for partner prediction and planning, which in turn supports strong cooperative performance as model capacity grows.

\section{Conclusion}
In this paper, we have proposed \emph{MetaMind}, a generalized cognitive world model with Meta-ToM for multi-agent systems, which learns from an ego agent's experience and generalizes to interactive environments in a zero-shot manner. We have developed Meta-ToM as a self-supervised inverse-inference paradigm that treats other agents as intentional entities and actively reasons latent beliefs and goals from behavioral trajectories. Meta-ToM combines self-reflection to learn goal-conditioned causal mechanisms, enables analogical inference to generalize mental-state estimation across agents, and constructs permutation-invariant collective beliefs for long-horizon planning under strategic adaptation. Moreover, we have provided theoretical guarantees for goal identification from behaviors, thereby offering principled conditions on separability and horizon length for robust inverse inference. Extended simulations across SMAC tasks show that MetaMind achieves ~54\% higher win rate than exsiting decentralized multi-agent world models under limited training steps, and a 2.4-fold improvement in few-shot generalization than exsiting centralized multi-agent world models.



\nocite{wu2020too,moreno2021neural,wang2018prefrontal,lee2022reasoning,li2026puzzle,wang2025leveraging,deihim2025transformer}

\nocite{kang2025viki,zhangad,wang2025rotate,wang2024making,zhang2024combo,hu2025agent2world,zhang2025revisiting,zhao2025learning,zhou2024dino,huang2025pigdreamer,wang2025dual}

\bibliography{example_paper}
\bibliographystyle{icml2026}

\newpage

\appendix
\onecolumn

\DoToC

\newpage

\section{Implementation Details}
\subsection{Multi-Q Network with EMA}
\label{ap:multiq-ema}
The use of a multi-Q Network with an exponential moving average (EMA) provides more conservative and stable value targets, which can effectively reduce overestimation bias and variance during temporal-difference learning \citep{hansen2023td}.
We maintain an ensemble of $N_Q$ parallel Q-functions $\{Q_{\varpi^{(i)}}\}_{i=1}^{N_Q}$, each with parameters $\varpi^{(i)}$ and EMA target copy $\bar{\varpi}^{(i)}$. The EMA targets are updated with momentum $\mu \in (0,1)$:
\begin{equation}
\bar{\varpi}^{(i)} \leftarrow \mu\, \bar{\varpi}^{(i)} + (1-\mu)\, \varpi^{(i)}, 
\qquad i=1,\dots,N_Q.
\label{eq:ema}
\end{equation}

At training step $t$, we form a conservative TD target by uniformly sampling a subset $\mathcal{M}_t \subset \{1,\dots,N_Q\}$ with $|\mathcal{M}_t|=2$, evaluating the EMA targets, and taking the minimum:
\begin{equation}
q_t 
= r_t 
+ \lambda\, \gamma \,
\min_{j \in \mathcal{M}_t} 
Q_{\bar{\varpi}^{(j)}}\!\left(
\tilde{\boldsymbol{b}}_{t+1},\;
\pi_{\varpi}\!\left(\tilde{\boldsymbol{b}}_{t+1}, \boldsymbol{g}\right)
\right),
\label{eq:td-target}
\end{equation}
where $\lambda \in (0,1]$ is a constant factor, $\gamma \in (0,1)$ is the discount, and $\tilde{\boldsymbol{b}}_{t+1}$ is the latent belief state predicted from $(\boldsymbol{b}_t,\boldsymbol{a}_t,\boldsymbol{g})$ by the model in the main body. The policy prior is $\tilde{\boldsymbol{a}}_{t+1}=\pi_{\varpi}(\boldsymbol{b}_t,\boldsymbol{g})$. This stochastic subsampling-and-minimization yields a conservative bootstrap that mitigates overestimation while remaining computationally efficient.
By using the TD target $q_t$ from (\ref{eq:td-target}), the per-network value loss is
\begin{equation}
\mathcal{L}_{Q}\!\left(\varpi^{(i)}\right) 
= \mathbb{E}_{\tau}\!\left[ 
\sum_{t=0}^{T} \zeta^{t}\;
\Im\!\left(\mathbf{s}^{(i)}_t,\; q_t\right)
\right],
\qquad i=1,\dots,N_Q,
\label{eq:per-q-loss}
\end{equation}
where $\zeta \in (0,1]$ is the temporal weighting factor, $\tau=\{\boldsymbol{o}_t,\boldsymbol{a}_t,r_t,\boldsymbol{o}_{t+1}\}$ is a sampled trajectory, and $\mathbf{s}^{(i)}_t$ are the logits produced by the $i$-th Q head at $(\boldsymbol{b}_t,\boldsymbol{a}_t,\boldsymbol{g})$. 

\subsection{Discrete Regression Objective in Log-Space.}
\label{ap:log}
Following the main body, both reward and value are trained via discrete regression \citep{hansen2023td} in a log-transformed space using a soft cross-entropy $\Im(\cdot)$ to ensure scale-invariant, multi-goal learning. Concretely, let the value head of the $i$-th Q-network output logits 
$\mathbf{s}^{(i)}_t \in \mathbb{R}^{K}$ over fixed log-space bins 
$\{b_k\}_{k=1}^{K}$. 
Let $\sigma(\cdot)$ denote the softmax and let $\mathbf{w}(y)\in\Delta^{K-1}$ be a soft target distribution induced by the scalar $y$, i.e., by linear interpolation between neighboring bins in log space. Then, the process can be given by:
\begin{align}
\Im\!\left(\mathbf{s}, y\right) = - \sum_{k=1}^{K} w_k(y)\, \log \sigma(\mathbf{s})_k, \quad
\tilde{y}(\mathbf{s}) = \sum_{k=1}^{K} \sigma(\mathbf{s})_k \, b_k,
\end{align}
where $\tilde{y}(\mathbf{s})$ is the differentiable expectation used for reporting or for auxiliary terms in the prediction loss (\ref{eq:main_loss}).

\subsection{Aggregation Approaches for Collective Belief}
\label{ap:agg}
We aggregate the unordered neighbor set by
\begin{equation}
\label{eq:st-init}
X^{(0)}=\{\boldsymbol{x}_{t}^{(i,j)}:\; j\in\mathcal{N}(i)\},
\qquad
\boldsymbol{x}_{t}^{(i,j)}=\big[\tilde{\boldsymbol{b}}_{t+1}^{(i,j)};\,\hat{\boldsymbol{g}}_{t}^{(i,j)}\big],
\end{equation}
by using $L$ permutation-invariant self-attention layers followed by an order-invariant pooling.
For each element $\boldsymbol{x}\in X^{(\ell)}$ and self-attention layer $\ell=0,\dots,L-1$, linear projections are defined by
\begin{equation}
Q_{\ell}(\boldsymbol{x})=\boldsymbol{x}W_{Q}^{(\ell)},\quad
K_{\ell}(\boldsymbol{x})=\boldsymbol{x}W_{K}^{(\ell)},\quad
V_{\ell}(\boldsymbol{x})=\boldsymbol{x}W_{V}^{(\ell)},
\end{equation}
where $W_{Q}^{(\ell)}$, $W_{K}^{(\ell)}$, and
$W_{V}^{(\ell)}$ are parameters.
The attention output for $\boldsymbol{x}$ is
\begin{equation}
\mathrm{Attn}_{\ell}(\boldsymbol{x},X^{(\ell)})
=\sum_{\boldsymbol{y}\in X^{(\ell)}}\alpha_{\ell}(\boldsymbol{x},\boldsymbol{y})\,V_{\ell}(\boldsymbol{y}),
\qquad
\alpha_{\ell}(\boldsymbol{x},\boldsymbol{y})
=\frac{\exp\!\big(Q_{\ell}(\boldsymbol{x})K_{\ell}(\boldsymbol{y})^{\!\top}/\sqrt{d_k}\big)}
{\sum_{\boldsymbol{z}\in X^{(\ell)}}\exp\!\big(Q_{\ell}(\boldsymbol{x})K_{\ell}(\boldsymbol{z})^{\!\top}/\sqrt{d_k}\big)},
\end{equation}
where $d_K$ represents the dimension of the key and query. 
Apply an output projection and a pointwise two-layer forward update:
\begin{equation}
\tilde{\boldsymbol{x}}
=\mathrm{Attn}_{\ell}(\boldsymbol{x},X^{(\ell)})\,W_{O}^{(\ell)},\qquad
\boldsymbol{h}
=\sigma\!\big(\tilde{\boldsymbol{x}}W_{1}^{(\ell)}+\boldsymbol{b}_{1}^{(\ell)}\big)W_{2}^{(\ell)}+\boldsymbol{b}_{2}^{(\ell)},
\end{equation}
with weights $W_{O}^{(\ell)}$, $W_{1}^{(\ell)}$,
$W_{2}^{(\ell)}$, biases $\boldsymbol{b}_{1}^{(\ell)}$,$\boldsymbol{b}_{2}^{(\ell)}$,
and nonlinearity $\sigma(\cdot)$.
Set
\begin{equation}
X^{(\ell+1)}=\{\boldsymbol{h}\,:\,\boldsymbol{x}\in X^{(\ell)}\}.
\end{equation}
Since attention weights and summations range over the entire set, this layer is permutation-invariant with respect to the ordering of $X^{(\ell)}$.
From the final set $X^{(L)}$, scalar scores and pooling weights are computed, respectively, by
\begin{equation}
s(\boldsymbol{x})=\boldsymbol{x}\boldsymbol{w}_{s}+c_{s},\qquad
\beta_{ij}=\frac{\exp\!\big(s(\boldsymbol{x}_{t}^{(i,j)})\big)}{\sum_{k\in\mathcal{N}(i)}\exp\!\big(s(\boldsymbol{x}_{t}^{(i,k)})\big)},
\end{equation}
where $\boldsymbol{w}_{s}\in\mathbb{R}^{d}$ and $c_{s}\in\mathbb{R}$.
The neighbor features are aggregated through a linear map
\begin{equation}
u(\boldsymbol{x})=\sigma(\boldsymbol{x}W_{u}+\boldsymbol{b}_{u}),\qquad
\boldsymbol{z}_{i}=\sum_{j\in\mathcal{N}(i)}\beta_{ij}\,u\!\big(\boldsymbol{x}_{t}^{(i,j)}\big),
\end{equation}
with parameters $W_{u}$ and $\boldsymbol{b}_{u}$.
Finally, the collective belief is projected by the pooled vector:
\begin{equation}
\label{eq:st-out}
\mathcal{E}_{\varrho}^{i}\!\big(\{\boldsymbol{x}_{t}^{(i,j)}\}\big)
=\sigma\!\Big(\,[\boldsymbol{x}_{t}^{(i,i)};\,\boldsymbol{z}_{i}]\,W_{o}+\boldsymbol{b}_{o}\Big),
\end{equation}
where parameters $W_{o}$ and $\boldsymbol{b}_{o}$.
By construction, the aggregation (\ref{eq:st-out}) is invariant to neighbor indices.
All weights and biases above are part of $\varrho$.

\subsection{Few-Shot Learning for Inverse Belief Inference}
\label{sec:fewshot-inv-belief}
\begin{figure}[h!]
    \centering
    \includegraphics[width=0.5\linewidth]{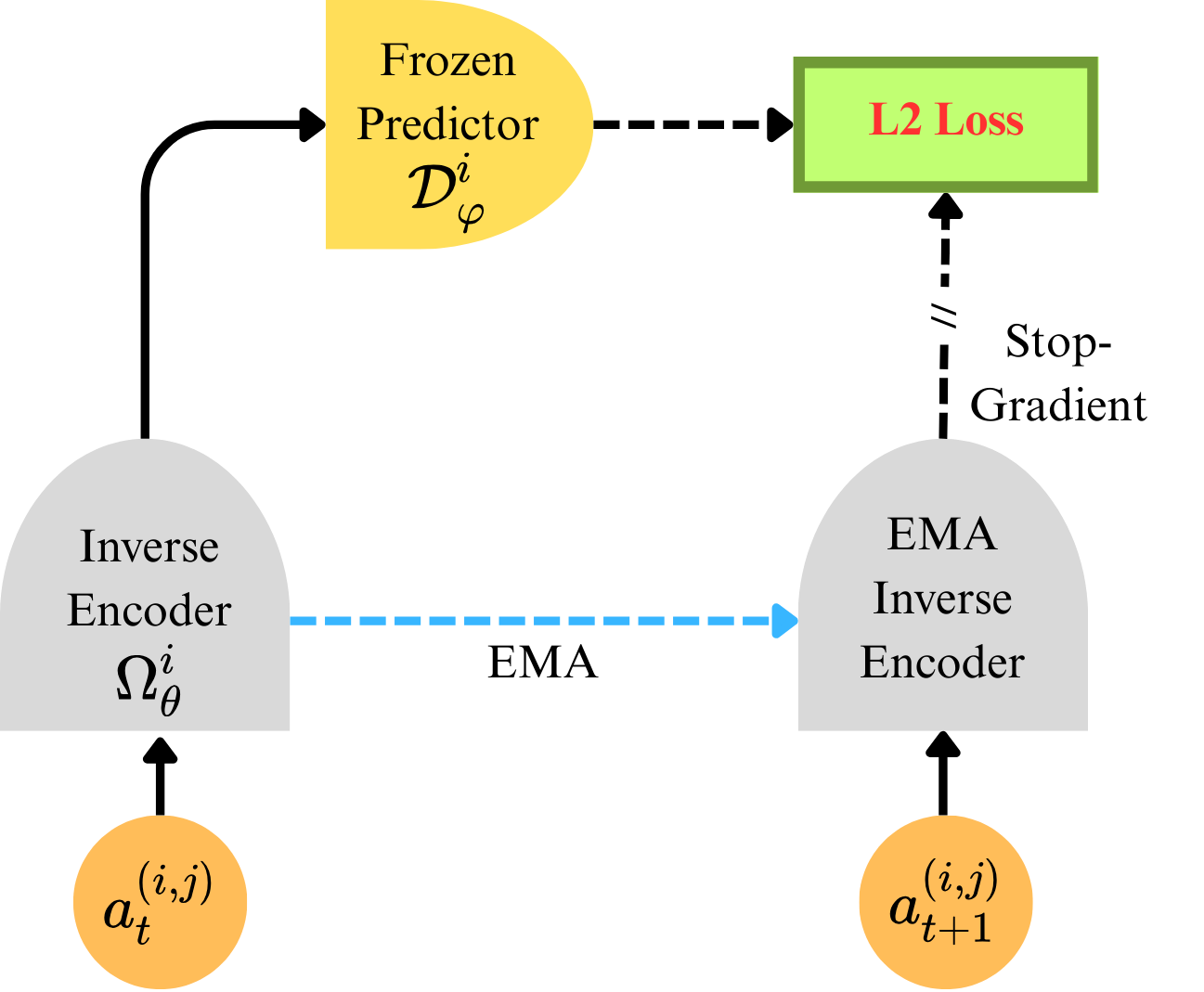}
    \caption{Few-shot learning for inverse belief inference by a self-supervised method in multi-agent systems.}
    \label{fig:fewshotfig}
\end{figure}
We propose a self-supervised, few-shot adaptation approach, with a similar architecture to joint-embedding
predictive architecture proposed in \citep{assran2025v}, that personalizes the inverse belief encoder of agent $i$ to any new partner $j$ by using only short behavioral trajectories
$\{\boldsymbol{a}_t^{(i,j)}\}$ and without supervised labels or explicit communication, as shown in Figure \ref{fig:fewshotfig}. Particularly,
the student inverse encoder $\Omega_\theta^{i}$ maps the current action to a latent belief $\hat{\boldsymbol{b}}^{(i,j)}_{t} \;=\; \Omega_\theta^{i}\!\big(\boldsymbol{a}_{t}^{(i,j)},\boldsymbol{g}_{t}^{(i,j)}\big)$, as given in (\ref{eq:inv}),
which is fed into a frozen predictor $\mathcal D_\varphi^{i}$ to generate a one-step belief prediction:
\begin{equation}
\tilde{\boldsymbol{b}}_{t+1}^{(i,j)} \;=\; \mathcal D_\varphi^{i}\!\big(\hat{\boldsymbol{b}}^{(i,j)}_{t},\boldsymbol{a}^{(i,j)}_t \hat{\boldsymbol{g}}^{(i,j)}_{t}\big),
\quad \text{with }\ \mathrm{sg}\left(\mathcal D_\varphi^{i}\right).
\end{equation}
Learning is driven by an action-space temporal consistency objective, i.e., L2 regression to the next observed action:
\begin{equation}
\label{eq:fewshot-inv-belief-loss}
\mathcal L_{\text{inv-belief}}^{(i,j)}
\;=\; \sum_{t} \big\| \mathcal D_\varphi^{i}\!\big(\Omega_\theta^{i}(\boldsymbol{a}_t^{(i,j)},\hat{\boldsymbol{g}}^{(i,j)}_{t})\big) - \bar{\Omega}_\theta^{i}(\boldsymbol{a}_{t+1}^{(i,j)},\hat{\boldsymbol{g}}^{(i,j)}_{t+1}) \big\|_2^2.
\end{equation}
where we maintain an EMA target encoder $\bar\Omega^{i}$ as the teacher encoder to stabilize few-shot updates and prevent representation drift as
\begin{equation}
\label{eq:ema-update}
\bar\Omega^{i} \;\leftarrow\; \tau\,\bar\Omega^{i} + (1-\tau)\,\Omega_\theta^{i},
\quad \tau\in[0,1).
\end{equation}

\newpage

\section{Algorithms}
\label{ap:alg}
The training process of the proposed MetaMind involves two stages. At Stage One, the MetaMind of each agent is trained in the single-agent settings to learn the classical world model-based forward inference model and Meta-ToM-based inverse inference model. At Stage Two, MetaMind of each agent generalizes to multi-agent settings in a zero-shot and scalable manner, where Meta-ToM can be directly used to infer others' internal mental states by analogical reasoning, and learn collective beliefs in a self-supervised manner.
The training details of Stage One and Stage Two can be summarized as shown in Algorithm \ref{alg:2} and Algorithm \ref{alg:3}. Since the MPC-based planner is typical and reused in Algorithms \ref{alg:2} and \ref{alg:3}, we present it separately as Algorithm \ref{alg:1}.

  \begin{algorithm}[H]
    \caption{A MPC-Based Planner.}
    \label{alg:1}
 \begin{algorithmic}
     \STATE Compute $ \boldsymbol{b}_t = \mathcal{B}_\varphi\left(\boldsymbol{o}_{t}, \boldsymbol{g} \right)$ and policy prior $\tilde{\boldsymbol{a}}_t = \pi_\varphi \left(\boldsymbol{b}_t, \boldsymbol{g}\right)$.
    \STATE Initialize ($\mu_t$, $\sigma_t$) from previous solution or $\tilde{\boldsymbol{a}}_t$.
    \STATE Sample candidate actions sequence $\left\{\boldsymbol{a}_{t:t+H}^{(v)}\right\}_{j=1}^J \sim \mathcal{N}(\mu_t, \operatorname{diag}(\sigma_t^2))$.
    \FOR{Iteration $u \rightarrow U$}
    \FOR{Candidate sequence $v \rightarrow V$}
    \STATE Imagine trajectories $\tilde{\boldsymbol{b}}_{t+1:t+H}^{(v)} = \mathcal{D}_{\varphi}(\boldsymbol{b}_t, \boldsymbol{a}_{t:t+H}^{(v)}, \boldsymbol{g})$.
    \STATE Estimate $R^{(v)} = \mathbb{E}_{\boldsymbol{a}_{t:t+H}}  \left[\sum_{h=0}^{H-1}  \gamma^h \mathcal{R}_{\varphi}\left(\tilde{\boldsymbol{b}}^{(v)}_{t+h}, \boldsymbol{a}^{(v)}_{t+h},\boldsymbol{g}\right)+\gamma^H \mathcal{Q}_{\varphi}\left(\tilde{\boldsymbol{b}}^{(v)}_{t+H},\boldsymbol{a}^{(v)}_{t+H},\boldsymbol{g}\right)\right].$
    \ENDFOR
    \STATE Update ($\mu_t$, $\sigma_t$) to maximize $\{R^{(v)}\}$
    \ENDFOR
    \STATE \textbf{Output:} The optimized action $\boldsymbol{a}_{t}^{*} \sim \mathcal{N}(\mu^*_t, \sigma^*_t)$.
  \end{algorithmic}
 \end{algorithm}
 
\begin{algorithm}[H]
    \caption{Self-Supervised Meta-ToM Training in Single-Agent Settings.}
    \label{alg:2}
 \begin{algorithmic}
    \STATE Set hyper-parameters $N$, $B$, $C$, $L$, $H$, $U$, $V$, and $\eta$.
    \STATE Initialize dataset $\mathcal{D}$ with $S$ seed episodes. 
    \STATE Initialize model parameters $\varphi$, $\theta$ and $\varrho$ for forward, inverse and aggregation model, respectively.
    \FOR{Training step $n \rightarrow N$}
    \FOR{Collect interval $c \rightarrow C$}
    \STATE \textcolor{softgreen}{// Forward Model Training}
    \STATE Sample $B$ Sequences $\{(\boldsymbol{o}_t, \boldsymbol{a}_t, r_t, \boldsymbol{g})\}_{t=k}^{k+L} \sim \mathcal{D}$.
    \STATE Compute $ \boldsymbol{b}_t = \mathcal{B}_\varphi\left(\boldsymbol{o}_{t}, \boldsymbol{g} \right)$.
    \STATE Predict $\tilde{\boldsymbol{b}}_t = \mathcal{D}_\varphi \left(\boldsymbol{b}_t, \boldsymbol{a}_t, \boldsymbol{g}\right)$, $\tilde{r} = \mathcal{R}_{\varphi}(\boldsymbol{b}_t,\boldsymbol{a}_t,\boldsymbol{g})$, and $\tilde{q} = \mathcal{Q}_{\varphi}(\boldsymbol{b}_t,\boldsymbol{a}_t,\boldsymbol{g})$.
    \STATE Update the forward model $\varphi \gets \varphi - \eta_\varphi \nabla_\varphi \mathcal{L}_{\text{pred}}(\varphi)$.
    \STATE \textcolor{softgreen}{// Inverse Model Training}
    \STATE Recover $\hat{\boldsymbol{g}}_{t}=\Psi_\theta(\boldsymbol{a}_{<t},\hat{\boldsymbol{g}}_{t-1})$ and $\hat{\boldsymbol{b}}_t= \Omega_\theta(\boldsymbol{a}_t,\hat{\boldsymbol{g}}_k)$.
    \STATE Update the inverse model $\theta \gets \theta - \eta_\theta \nabla_\theta \mathcal{L}_{\text{inv}}(\theta)$.
    \ENDFOR
    \STATE \textcolor{blue}{// Real Environment Interaction \& Data Collection}
    \STATE Start a environment with a random goal $\boldsymbol{g}$ by $env.reset()$.
    \FOR{Time step $t \rightarrow T$}
    \STATE \textcolor{softgreen}{// MPC Decision Making}
    \STATE Obtain the optimal action $\boldsymbol{a}^*_t$ from Algorithm \ref{alg:3}.
    \STATE Collect real data $r_t, \boldsymbol{o}_{t+1} \gets env.step(\boldsymbol{a}^*_t,\boldsymbol{g})$.
    \ENDFOR
    \STATE Add experience to dataset $\mathcal{D} \gets \mathcal{D} \cup \{(\boldsymbol{o}_t, \boldsymbol{a}_t, r_t,\boldsymbol{g})\}_{t=1}^T$.
    \ENDFOR
 \end{algorithmic}
 \end{algorithm}
 
\newpage

 \begin{algorithm}[H]
    \caption{MetaMind in Multi-Agent Settings.}
    \label{alg:3}
 \begin{algorithmic}
    \STATE Set hyper-parameters $N$, $B$, $C$, $L$, $H$, $U$, $V$, and $\eta$.
    \STATE Initialize dataset $\mathcal{D}$ with $S$ seed episodes. 
    \STATE Initialize model parameters $\varphi$, $\theta$ and $\varrho$ for forward, inverse and aggregation model, respectively.
    \FOR{Training step $n \rightarrow N$}
    \FOR{Collect interval $c \rightarrow C$}
    \STATE \textcolor{softgreen}{// Forward Model}
    \STATE Sample $B$ Sequences $\{(\boldsymbol{o}_t, \boldsymbol{a}_t, r_t, \boldsymbol{g})\}_{t=k}^{k+L} \sim \mathcal{D}$.
    \STATE Compute $ \boldsymbol{b}_t = \mathcal{B}_\varphi\left(\boldsymbol{o}_{t}, \boldsymbol{g} \right)$.
    \STATE Predict $\tilde{\boldsymbol{b}}_t = \mathcal{D}_\varphi \left(\boldsymbol{b}_t, \boldsymbol{a}_t, \boldsymbol{g}\right)$, $\tilde{r} = \mathcal{R}_{\varphi}(\boldsymbol{b}_t,\boldsymbol{a}_t,\boldsymbol{g})$, and $\tilde{q} = \mathcal{Q}_{\varphi}(\boldsymbol{b}_t,\boldsymbol{a}_t,\boldsymbol{g})$.
    \STATE \textcolor{softgreen}{// Inverse Model}
    \STATE Recover $\hat{\boldsymbol{g}}_{t}=\Psi_\theta(\boldsymbol{a}_{<t},\hat{\boldsymbol{g}}_{t-1})$ and $\hat{\boldsymbol{b}}_t= \Omega_\theta(\boldsymbol{a}_t,\hat{\boldsymbol{g}}_k)$.
    \STATE \textcolor{softgreen}{// Collective Belief Training (Agent $i$)}
    \STATE Form collective belief by $\mathcal{E}_{\varrho}^i(\{\tilde{\boldsymbol{b}}_{t+1}^{(i,j)},\hat{\boldsymbol{g}}_{t}^{(i,j)}\})$
    \STATE Update the aggregation model $\varrho \gets \varrho - \eta_\varrho \nabla_\varrho \mathcal{L}_{\text{col}}(\varrho)$.
    \ENDFOR
    \STATE \textcolor{blue}{// Real Environment Interaction \& Data Collection}
    \STATE Start a environment with a random goal $\boldsymbol{g}$ by $env.reset()$.
    \FOR{Time step $t \rightarrow T$}
    \STATE \textcolor{softgreen}{// MPC Decision Making (Agent $i$)}
    \STATE Obtain imagined multi-agent trajectories by $\mathcal{E}_{\varrho}^i(\{\tilde{\boldsymbol{b}}_{t+1}^{(i,j)},\hat{\boldsymbol{g}}_{t}^{(i,j)}\})$
    \STATE Obtain the optimal action $\boldsymbol{a}^*_t$ from Algorithm \ref{alg:1}.
    \STATE Collect real data $r_t, \boldsymbol{o}_{t+1} \gets env.step(\boldsymbol{a}^*_t,\boldsymbol{g})$.
    \ENDFOR
    \STATE Add experience to dataset $\mathcal{D} \gets \mathcal{D} \cup \{(\boldsymbol{o}_t, \boldsymbol{a}_t, r_t,\boldsymbol{g})\}_{t=1}^T$.
    \ENDFOR
 \end{algorithmic}
 \end{algorithm}


\section{Proof}
\subsection{Proof of Theorem \ref{thm:M2}: Goal Identification}
\label{proof:a2}
\begin{proof}
We first prove $d_H(\boldsymbol{g}^*)<d_H(\boldsymbol{g}^k)$ for all $k\neq k^*$. Based on Definition~\ref{def1},
for any $\boldsymbol{g}$, we have the behavior-only TD predictor and residual as
\begin{equation}
\label{eq:predres}
\hat{\boldsymbol{b}}_t(\boldsymbol{g}):=\Omega(\boldsymbol{a}_t,\boldsymbol{g}),\quad   \tilde{\boldsymbol{b}}_{t+1}(\boldsymbol{g}):=\mathcal{D}(\hat{\boldsymbol{b}}_t(\boldsymbol{g}),\boldsymbol{a}_t,\boldsymbol{g}),  
\end{equation}
and 
\begin{equation}
\label{eq:predres2}
\tilde{\boldsymbol{a}}_{t+1}(\boldsymbol{g}):=\pi(\tilde{\boldsymbol{b}}_{t+1}(\boldsymbol{g}),\boldsymbol{g}),\quad
\phi_t(\boldsymbol{g}):=\boldsymbol{a}_{t+1}-\tilde{\boldsymbol{a}}_{t+1}(\boldsymbol{g}),
\end{equation}
and stack 
\begin{equation}
\Phi_H(\boldsymbol{g}):=[\phi_0(\boldsymbol{g});\dots;\phi_{H-2}(\boldsymbol{g})].
\end{equation}
As defined in (\ref{eq:predres2}) at $\boldsymbol{g}^*$, we have
\begin{equation}
\phi_t(\boldsymbol{g}^*)=\boldsymbol{a}_{t+1}-\pi\!\big(\mathcal{D}(\Omega(\boldsymbol{a}_t,\boldsymbol{g}^*),\boldsymbol{a}_t,\boldsymbol{g}^*),\boldsymbol{g}^*\big).
\end{equation}
Given $\boldsymbol{a}_{t+1}=\pi^j(\boldsymbol{b}_{t+1}^j$, $\boldsymbol{g}^*)$ and $\boldsymbol{b}_{t+1}^j=\mathcal{D}^j(\boldsymbol{b}_t^j,\boldsymbol{a}_t,\boldsymbol{g}^*)$, $\phi_t(\boldsymbol{g}^*)$ can be written as
\begin{align}
\label{eq:M2decA}
\phi_t(\boldsymbol{g}^*)
&=\underbrace{\pi^j(\boldsymbol{b}_{t+1}^j,\boldsymbol{g}^*)-\pi(\boldsymbol{b}_{t+1}^j,\boldsymbol{g}^*)}_{\Delta\pi_{t+1}}
+\ \pi\!\big(\mathcal{D}^j(\boldsymbol{b}_t^j,\boldsymbol{a}_t,\boldsymbol{g}^*),\boldsymbol{g}^*\big)
-\ \pi\!\big(\mathcal{D}(\Omega(\boldsymbol{a}_t,\boldsymbol{g}^*),\boldsymbol{a}_t,\boldsymbol{g}^*),\boldsymbol{g}^*\big)\nonumber\\
&=\Delta\pi_{t+1}
+\ \Big(\pi(\mathcal{D}^j(\boldsymbol{b}_t^j,\boldsymbol{a}_t,\boldsymbol{g}^*),\boldsymbol{g}^*)-\pi(\mathcal{D}(\boldsymbol{b}_t^j,\boldsymbol{a}_t,\boldsymbol{g}^*),\boldsymbol{g}^*)\Big) \nonumber\\
&\qquad-\ \Big(\pi(\mathcal{D}(\Omega(\boldsymbol{a}_t,\boldsymbol{g}^*),\boldsymbol{a}_t,\boldsymbol{g}^*),\boldsymbol{g}^*)-\pi(\mathcal{D}(\boldsymbol{b}_t^j,\boldsymbol{a}_t,\boldsymbol{g}^*),\boldsymbol{g}^*)\Big).
\end{align}
Since the per-step mismatches are given by
$\Delta \boldsymbol{b}^j_t:=\Omega(\boldsymbol{a}_t,\boldsymbol{g}^*)-\boldsymbol{b}_t^j$, $\Delta \mathcal{D}_t(\boldsymbol{b}):=\mathcal{D}(\boldsymbol{b},\boldsymbol{a}_t,\boldsymbol{g}^*)-\mathcal{D}^j(\boldsymbol{b},\boldsymbol{a}_t,\boldsymbol{g}^*)$, and $
\Delta\pi_{t+1}:=\pi^j(\boldsymbol{b}_{t+1}^j,\boldsymbol{g}^*)-\pi(\boldsymbol{b}_{t+1}^j,\boldsymbol{g}^*)$, the last item in (\ref{eq:M2decA}) equals
$\pi(\mathcal{D}(\boldsymbol{b}_t^j+\Delta \boldsymbol{b}^j_t,\boldsymbol{a}_t,\boldsymbol{g}^*),\boldsymbol{g}^*)-\pi(\mathcal{D}(\boldsymbol{b}_t^j,\boldsymbol{a}_t,\boldsymbol{g}^*),\boldsymbol{g}^*)$, hence
\begin{align}
\phi_t(\boldsymbol{g}^*)
&=\Delta\pi_{t+1}
+\ \Big(\pi(\mathcal{D}^j(\boldsymbol{b}_t^j,\boldsymbol{a}_t,\boldsymbol{g}^*),\boldsymbol{g}^*)-\pi(\mathcal{D}(\boldsymbol{b}_t^j,\boldsymbol{a}_t,\boldsymbol{g}^*),\boldsymbol{g}^*)\Big)\nonumber\\
&\qquad-\ \Big(\pi(\mathcal{D}(\boldsymbol{b}_t^j+\Delta \boldsymbol{b}^j_t,\boldsymbol{a}_t,\boldsymbol{g}^*),\boldsymbol{g}^*)-\pi(\mathcal{D}(\boldsymbol{b}_t^j,\boldsymbol{a}_t,\boldsymbol{g}^*),\boldsymbol{g}^*)\Big).
\label{eq:M2exact}
\end{align}
(\ref{eq:M2exact}) provides an exact identity.
By the $\xi_\pi$-Lipschitzness of $\pi(\cdot,\boldsymbol{g}^*)$ in $\boldsymbol{b}$,
\begin{equation}
\begin{aligned}
\label{eq:M2L1}
\Big\|\pi(\mathcal{D}^j(\boldsymbol{b}_t^j,\boldsymbol{a}_t,\boldsymbol{g}^*),\boldsymbol{g}^*)-\pi(\mathcal{D}(\boldsymbol{b}_t^j,\boldsymbol{a}_t,\boldsymbol{g}^*),\boldsymbol{g}^*)\Big\|
&\le \xi_\pi\ \|\mathcal{D}^j(\boldsymbol{b}_t^j,\boldsymbol{a}_t,\boldsymbol{g}^*)-\mathcal{D}(\boldsymbol{b}_t^j,\boldsymbol{a}_t,\boldsymbol{g}^*)\| \\
& = \xi_\pi\,\|\Delta \mathcal{D}_t(\boldsymbol{b}_t^j)\|. 
\end{aligned}
\end{equation}
By using $\xi_\pi$ again and then using the $\xi_{\mathcal{D}}$-Lipschitzness of $\mathcal{D}(\cdot,\boldsymbol{a}_t,\boldsymbol{g}^*)$ in $\boldsymbol{b}$,
\begin{align}
\Big\|\pi(\mathcal{D}(\boldsymbol{b}_t^j+\Delta \boldsymbol{b}^j_t,\boldsymbol{a}_t,\boldsymbol{g}^*),\boldsymbol{g}^*)-\pi(\mathcal{D}(\boldsymbol{b}_t^j,\boldsymbol{a}_t,\boldsymbol{g}^*),\boldsymbol{g}^*)\Big\|
&\le \xi_\pi \ \| \mathcal{D}(\boldsymbol{b}_t^j+\Delta \boldsymbol{b}^j_t,\boldsymbol{a}_t,\boldsymbol{g}^*)-\mathcal{D}(\boldsymbol{b}_t^j,\boldsymbol{a}_t,\boldsymbol{g}^*)\|\nonumber\\
&\le \xi_\pi \xi_{\mathcal{D}}\ \|\Delta \boldsymbol{b}^j_t\|. \label{eq:M2L2}
\end{align}

Take norms in (\ref{eq:M2exact}) and apply the triangle inequality together with
(\ref{eq:M2L1})-(\ref{eq:M2L2}), then
\begin{equation}
\label{eq:M2rt}
\|\phi_t(\boldsymbol{g}^*)\|\ \le\ \underbrace{\|\Delta\pi_{t+1}\|}_{\text{policy mismatch}}+\,\xi_\pi\big(\underbrace{\|\Delta \mathcal{D}_t(\boldsymbol{b}_t^j)\|}_{\text{dynamics mismatch}}+\,\xi_{\mathcal{D}}\underbrace{\|\Delta \boldsymbol{b}^j_t\|}_{\text{belief mismatch}} \big) := \alpha_t.
\end{equation}
Therefore, we can obtain
\begin{equation}
\label{eq:M2stack}
d_H(\boldsymbol{g}^*)=\sum_{t=0}^{H-2}\|\phi_t(\boldsymbol{g}^*)\|^2
\ \le\ \sum_{t=0}^{H-2}\alpha_t^2
\ =\ \delta_{\mathrm{act}}^{\,2}
\quad\Rightarrow\quad
d_H(\boldsymbol{g}^*)\ \le\ \delta_{\mathrm{act}}.
\end{equation}
Given the empirical residual gap $\bar\gamma_H(\mathcal C;\boldsymbol{g}^*):= \min_{k\neq k^*}\ \|\Phi_H(\boldsymbol{g}^k)-\Phi_H(\boldsymbol{g}^*)\|_2$ defined in Definition~\ref{def3}, $\gamma_H(\mathcal C;\boldsymbol{g}^*)\ :=\ \min_{k\neq k^*}\ \|\Phi_H(\boldsymbol{g}^k)\|_2$ and $d_H(\boldsymbol{g}^*)\le\delta_{\mathrm{act}}$ given in (\ref{eq:M2stack}), for any $k\neq k^*$, we have
\begin{equation}
\bar\gamma_H(\mathcal C;\boldsymbol g^*)
\ \ge\ \gamma_H(\mathcal C;\boldsymbol{g}^*)-d_H(\boldsymbol{g}^*)
\ \ge\ \gamma_H(\mathcal C;\boldsymbol{g}^*)-\delta_{\mathrm{act}}.
\end{equation}
Hence, under $\delta_{\mathrm{act}}<\gamma/2$,
\begin{equation}
\bar\gamma_H(\mathcal C;\boldsymbol g^*)\ \ge\ \gamma_H(\mathcal C;\boldsymbol{g}^*)-\delta_{\mathrm{act}}\ >\ \gamma_H(\mathcal C;\boldsymbol{g}^*)/2\ >\ 0.
\end{equation}
Therefore for every $k\neq k^*$,
\begin{equation}
d_H(\boldsymbol g^k)\ \ge\ d_H(\boldsymbol{g}^*)+\bar\gamma_H(\mathcal C;\boldsymbol g^*)\ >\ d_H(\boldsymbol{g}^*),
\end{equation}
thus, $\boldsymbol g^*$ is the unique minimizer of $k\mapsto d_H(\boldsymbol g^k)$, and the nearest-neighbor estimator returns $\hat{\boldsymbol g}=\boldsymbol g^*$. Theorem \ref{thm:M2} is proved.
\end{proof}

\subsection{Proof of Lemma \ref{lemma:local-margin}: Local Linear Margin}
\label{proof:l1}
\begin{proof}
Fix $k\neq k^*$ and set $\Delta_k:=\boldsymbol{g}^k-\boldsymbol{g}^{k^*}$.
For each coordinate $\Phi_H$, by the fundamental theorem of calculus,
$\Phi_H(\boldsymbol{g}^k)-\Phi_H(\boldsymbol{g}^{k^*})=\int_0^1 \nabla \Phi_H(\boldsymbol{g}^{k^*}+s\Delta_k)^\top\Delta_k\,ds$.
Stacking all coordinates yields
\begin{equation}
\label{eq:meanvalue}
\Phi_H(\boldsymbol{g}^k)-\Phi_H(\boldsymbol{g}^{k^*})=\Big(\int_0^1 J_\Phi(\boldsymbol{g}^{k^*}+s\Delta_k)\,ds\Big)\Delta_k
=:J_{\rm avg}(\boldsymbol{g}^{k^*},\boldsymbol{g}^k)\,\Delta_k.
\end{equation}
Since $\Phi_H(\boldsymbol{g}^{k^*})=0$, we have $\Phi_H(\boldsymbol{g}^k)=J_{\rm avg}\Delta_k$.

Assume $J_\Phi$ is $L_J$-Lipschitz near $\boldsymbol{g}^*$. Then
\begin{align}
\big\|J_{\rm avg}-J_\Phi(\boldsymbol{g}^*)\big\|
&=\Big\|\int_0^1 \big(J_\Phi(\boldsymbol{g}^{k^*}+s\Delta_k)-J_\Phi(\boldsymbol{g}^*)\big)\,ds\Big\| \nonumber \\
&\le \int_0^1 \big\|J_\Phi(\boldsymbol{g}^{k^*}+s\Delta_k)-J_\Phi(\boldsymbol{g}^*)\big\|\,ds \nonumber\\
&\le \int_0^1 L_J s\|\Delta_k\|\,ds = \frac{L_J}{2}\,\|\Delta_k\|.
\label{eq:JavgLip}
\end{align}

Applying the Weyl's perturbation inequality, i.e., for any matrices $A,B$, $\sigma_{\min}(A+B)\ge \sigma_{\min}(A)-\|B\|$, to $A=J_\Phi(\boldsymbol{g}^*)$ and $B=J_{\rm avg}-J_\Phi(\boldsymbol{g}^*)$, we have
\begin{equation}
\sigma_{\min}(J_{\rm avg})
\ge \sigma_{\min}(J_\Phi(\boldsymbol{g}^*))-\|J_{\rm avg}-J_\Phi(\boldsymbol{g}^*)\|
\stackrel{(\ref{eq:JavgLip})}{\ge} \sigma_{\min}(J_\Phi(\boldsymbol{g}^*))-\frac{L_J}{2}\|\Delta_k\|.
\end{equation}
If $\|\Delta_k\|\le \phi_\star:=\sigma_{\min}(J_\Phi(\boldsymbol{g}^*))/L_J$, then
\begin{equation}
\sigma_{\min}(J_{\rm avg})\ge \frac12\sigma_{\min}(J_\Phi(\boldsymbol{g}^*)).
\end{equation}
By using $\|Ax\|\ge \sigma_{\min}(A)\|x\|$ and (\ref{eq:meanvalue}), we have
\begin{equation}
\label{eq:M1margin}
d_H(\boldsymbol{g}^k)=\|J_{\rm avg}\Delta_k\|\ge \sigma_{\min}(J_{\rm avg})\,\|\Delta_k\|
\ge \frac12\,\sigma_{\min}(J_\Phi(\boldsymbol{g}^*))\,\|\boldsymbol{g}^k-\boldsymbol{g}^{k^*}\|,
\end{equation}
and (\ref{eq:M1margin}) follows by minimizing over $k\neq k^*$.
Lemma \ref{lemma:local-margin} is proved.
\end{proof}

\subsection{Proof of Theorem \ref{thm:a2}: Sufficient Horizon Size for Goal Identification}
\label{proof:a3}
\begin{proof}
We have the persistence-of-excitation defined in (\ref{eq:PE}) as
\begin{equation}
\sum_{t=k}^{k+m-1}\Big(\frac{\partial \phi_t}{\partial \boldsymbol{g}}(\boldsymbol{g}^*)\Big)^{\!\top}
\Big(\frac{\partial \phi_t}{\partial \boldsymbol{g}}(\boldsymbol{g}^*)\Big)\ \succeq\ \beta\, I_{d_g},
\end{equation}
at the true goal $\boldsymbol{g}^*$ and we have
\begin{equation}
G_t(\boldsymbol{g}^*):=\frac{\partial \boldsymbol \phi_t}{\partial \boldsymbol g}(\boldsymbol{g}^*),
\quad
W_H(\boldsymbol{g}^*):=\sum_{t=0}^{H-2} G_t(\boldsymbol{g}^*)^\top G_t(\boldsymbol{g}^*).
\end{equation}
Partition the index set $\{0,1,\ldots,H-2\}$ into
$q:=\big\lfloor\frac{H-1}{m}\big\rfloor$ disjoint windows
$\mathcal I_i=\{(i-1)m,\ldots,im-1\}$ for $i=1,\ldots,q$.
By (\ref{eq:PE}), for each window, we have
\begin{equation}
\sum_{t\in\mathcal I_i} G_t^\top G_t \succeq \beta I_{d_{\boldsymbol{g}}}.
\end{equation}
Then, adding these positive semi-definite inequalities over $i=1,\ldots,q$ obtains
\begin{equation}
\label{tag:1}
W_H(\boldsymbol{g}^*)=\sum_{t=0}^{H-2} G_t^\top G_t
=\sum_{i=1}^q\ \sum_{t\in\mathcal I_i} G_t^\top G_t
\succeq \sum_{i=1}^q \beta I_{d_g}
= q\,\beta\, I_{d_{\boldsymbol{g}}}.
\end{equation}

By the Rayleigh-Ritz characterization, we have
\begin{equation}
\label{tag:2}
\lambda_{\min}\!\big(W_H(\boldsymbol{g}^*)\big)
=\min_{\|z\|=1} z^\top W_H(\boldsymbol{g}^*) z
\ \stackrel{(\ref{tag:1})}{\ge}\ \min_{\|z\|=1} z^\top (q\beta I) z
= q\beta
= \Big\lfloor\frac{H-1}{m}\Big\rfloor\,\beta.
\end{equation}

Let $J_\Phi(\boldsymbol{g}^*)$ be the stacked Jacobian.
By taking the singular value decomposition $J_\Phi=U\Sigma V^\top$, then we have
$W_H=J_\Phi^\top J_\Phi=V\Sigma^2 V^\top$, and every eigenvalue of $W_H$ equals a squared singular value of $J_\Phi$.
Hence, we have
\begin{equation}
\label{tag:3}
\sigma_{\min}\!\big(J_\Phi(\boldsymbol{g}^*)\big)
=\sqrt{\lambda_{\min}\!\big(W_H(\boldsymbol{g}^*)\big)}.
\end{equation}
By the local constructive margin inequality (\ref{eq:M1margin}),
for each $k\neq k^*$,
\begin{equation}
d_H(\boldsymbol{g}^k)
\ \ge\ \frac12\,\sigma_{\min}\!\big(J_\Phi(\boldsymbol{g}^*)\big)\,
\|\boldsymbol{g}^k-\boldsymbol{g}^{k^*}\|.
\end{equation}
Let $\rho_{\min}:=\min_{k\neq \ell}\|\boldsymbol{g}^k-\boldsymbol{g}^\ell\|$.
Taking the minimum over $k\neq k^*$, we have
\begin{equation}
\label{tag:4}
\gamma_H(\mathcal C;\boldsymbol{g}^*)
:=\min_{k\neq k^*}\|\Phi_H(\boldsymbol{g}^k)\|
\ \ge\ \frac12\,\sigma_{\min}\!\big(J_\Phi(\boldsymbol{g}^*)\big)\,\rho_{\min}.
\end{equation}
Now substitute (\ref{tag:3}) into (\ref{tag:4}), and then apply (\ref{tag:2}), we have
\begin{equation}
\gamma_H(\mathcal C;\boldsymbol{g}^*)
\ \ge\ \frac12\,\rho_{\min}\,\sqrt{\lambda_{\min}\!\big(W_H(\boldsymbol{g}^*)\big)}
\ \ge\ \frac12\,\rho_{\min}\,\sqrt{\Big\lfloor\frac{H-1}{m}\Big\rfloor\,\beta}.
\end{equation}
Hence, a sufficient condition is given by
\begin{equation}
\delta_{\mathrm{act}}\ <\ \frac14\,\rho_{\min}\,\sqrt{\Big\lfloor\frac{H-1}{m}\Big\rfloor\,\beta}
\quad\Longleftrightarrow\quad
\Big\lfloor\frac{H-1}{m}\Big\rfloor\ >\ \Big(\frac{4\,\delta_{\mathrm{act}}}{\rho_{\min}\sqrt{\beta}}\Big)^{\!2}\;=\;T.
\end{equation}
To ensure the strict inequality on the left, it is enough to require
\begin{equation}
\Big\lfloor\frac{H-1}{m}\Big\rfloor\ \ge\ \big\lfloor T\big\rfloor+1,
\end{equation}
because $\lfloor T\rfloor+1> T$ for all real $T$.
Since $\lfloor x\rfloor\ge n$ iff $x\ge n$ for integer $n$, the above is equivalent to
\begin{equation}
\frac{H-1}{m}\ \ge\ \big\lfloor T\big\rfloor+1
\quad\Longleftrightarrow\quad
H\ \ge\ 1\;+\;m\Big(\big\lfloor T\big\rfloor+1\Big).
\end{equation}
Under this condition, we have $\delta_{\mathrm{act}}<\bar\gamma_H(\mathcal C;\boldsymbol g^*)/2$, so by Theorem~\ref{thm:M2}, the nearest-neighbor rule recovers the true goal. Corollary \ref{thm:a2} is proved.
\end{proof}

\begin{remark}
If $T\notin\mathbb{N}$, the simpler bound $H\ge 1+m\lceil T\rceil$ already ensures
$\lfloor\frac{H-1}{m}\rfloor\ge\lceil T\rceil> T$. When $T\in\mathbb{N}$, an extra window, replacing $\lceil T\rceil$ by $T+1$, is needed to guarantee strict inequality.
\end{remark}

\newpage

\section{Hyperparameters}\label{hyp}
 \subsection{Environment Setting}
\begin{table}[H]
\centering
\caption{Environment Setting for MARIE (Goal Definitions)}
\label{tab:env_setting}
\begin{tabular}{l|l|c|p{8.6cm}}
\hline
\textbf{Env} & \textbf{Scenario} & \textbf{\#Agents} & \textbf{Goal Setting} \\ \hline

\multirow{9}*{SMAC (Easy)} 
& 1c3s5z & 9 &
Type: Heterogeneous \newline
Roles: C$\times$1, S$\times$3, Z$\times$5 \newline
Goals: $G{=}3$  \\ \cline{2-4}

& 2m\_vs\_1z & 2 &
Type: Homogeneous \newline
Roles: M$\times$2 \newline
Goals: $G{=}1$  \\ \cline{2-4}

& 2s\_vs\_1sc & 2 &
Type: Homogeneous \newline
Roles: S$\times$2 \newline
Goals: $G{=}1$  \\ \cline{2-4}

& 2s3z & 5 &
Type: Heterogeneous \newline
Roles: S$\times$2, Z$\times$3 \newline
Goals: $G{=}2$ \\ \cline{2-4}

& 3m & 3 &
Type: Homogeneous \newline
Roles: M$\times$3 \newline
Goals: $G{=}1$ \\ \cline{2-4}

& 3s\_vs\_3z & 6 &
Type: Homogeneous \newline
Roles: S$\times$3 \newline
Goals: shared $G{=}1$  \\ \cline{2-4}

& 3s\_vs\_4z & 7 &
Type: Homogeneous \newline
Roles: S$\times$3 \newline
Goals: $G{=}1$ \\ \cline{2-4}

& 8m & 8 &
Type: Homogeneous \newline
Roles: M$\times$8 \newline
Goals: $G{=}1$ \\ \cline{2-4}

& MMM & 10 &
Type: Heterogeneous \newline
Roles: Md$\times$1, Mr$\times$2, Ma$\times$7 \newline
Goals: $G{=}3$ \\ \hline

\multirow{2}*{SMAC (Hard)} 
& 3s\_vs\_5z & 8 &
Type: Homogeneous \newline
Roles: S$\times$3 \newline
Goals: $G{=}1$ \\ \cline{2-4}

& 2c\_vs\_64zg & 2 &
Type: Homogeneous \newline
Roles: C$\times$2 \newline
Goals: $G{=}1$ \\ \hline

SMAC (SuperHard) 
& corridor & 6 &
Type: Homogeneous \newline
Roles: Z$\times$6 \newline
Goals: $G{=}1$ \\ \hline
\end{tabular}

\vspace{1mm}
\footnotesize{\textbf{Role abbreviations:} C=Colossus, S=Stalker, Z=Zealot, M=Marine, Md=Medivac, Mr=Marauder, Ma=Marine. $G$ denotes the number of goal categories.
The total number of goals is $\bar{G}=17$}
\end{table}

\newpage

\subsection{Model Setting}
The hyperparameters of models are presented in TABLE 3.
\begin{table}[H]
    \centering
    \caption{Hyperparameter Settings}
    \label{tb:tdmpc2-hyp}
    \begin{tabular}{l|c|c}
    \hline
    \textbf{Parameter} & \textbf{Symbol} & \textbf{Value} \\ \hline
    \multicolumn{3}{l}{\textbf{Shared Parameters}} \\ \hline
  Replay memory size & --- & 1e6 \\ 
Batch size & $B$ & 50 \\ 
Sequence length & $L$ & 64 \\ 
Seed episode & $S$ & 5 \\ 
Training episodes & $N$ & 1e3 \\
Collect Interval & $C$ & 100 \\
Max episode length & --- & 500 \\
Exploration noise & --- & 0.3 \\
Imagination horizon & $H$ & 30 \\ 
Gradient clipping & --- & 100 \\ \hline

    \multicolumn{3}{l}{\textbf{Planning (MPC)}} \\ \hline
    Iterations & --- & 40 \\
    Encoder dim & --- & 256 \\
    MLP dim & --- & 512 \\
    Latent state dim & --- & 512 \\
    Task embedding dim & --- & 96 \\
    Activation & --- & LayerNorm + Mish \\
    Q-functions & --- & 5 \\ 
    Learning rate & --- & $3\times10^{-4}$ \\
    Encoder LR & --- & $1\times10^{-4}$ \\
    Gradient clip & --- & 20 \\ \hline

    \multicolumn{3}{l}{\textbf{Inverse Inference (Meta-ToM) Module}} \\ \hline
    Goal latent dim & $d_g$ & 128 \\
    Belief latent dim & $d_b$ & 128 \\
    Hidden dim (MLP) & --- & 256 \\
    Number of layers & --- & 2 \\
    Activation & --- & ReLU \\
    Normalization & --- & LayerNorm \\
    Dropout & --- & 0.1 \\ 
    Inverse loss weight & $\lambda_{\text{inv}}$ & 0.5 \\
    Reflection loss weight & $\lambda_{\text{ref}}$ & 0.1 \\
    Optimizer & --- & Adam \\
    Learning rate & $\eta_{\text{inv}}$ & $1\times10^{-4}$ \\
    EMA teacher momentum & $\tau$ & 0.995 \\ \hline

\multicolumn{3}{l}{\textbf{VQ-VAE Tokenizer} (MARIE)} \\ \hline
Encoder/Decoder layers & --- & 3 \\
Hidden size & --- & 512 \\
Activation & --- & GELU \\
Codebook size & $|Z|$ & 512 \\
Tokens per observation & $K$ & 16 \\
Code dimension & --- & 128 \\
Commitment loss coef. & $\beta$ & 10.0 \\
Codebook update & --- & EMA \\
Tokenizer batch size & $B_{tok}$ & 256 \\
Tokenizer learning rate & $\eta_{tok}$ & 3e-4 \\
Tokenizer optimizer & --- & AdamW \\ \hline
\end{tabular}
\end{table}

\newpage

\section{Further Discussion}

\subsection{Effectiveness}
The proposed MetaMind demonstrates effectiveness since:

\paragraph{Meta-ToM Mechanism.} By introducing a self-supervised Meta-ToM mechanism, each agent develops a bidirectional inference loop between beliefs and actions, which turns passive trajectory modeling into active, goal-oriented reasoning based on causal behavioral principles.

\paragraph{Self-Reflection.} The self-reflection reduces the ambiguity inherent in inverse inference and provides a principled means for bounding estimation error, as formally established through the identification analysis. 

\paragraph{Meta-Cognition.} By generalizing introspective inference from single-agent trajectories to other agents by analogical reasoning, MetaMind exhibits a form of meta-cognition that enables zero-shot adaptability in heterogeneous multi-agent environments. A self-supervised approach can be used for learning accurate inverse belief inference and learning collective beliefs in multi-agent systems.

These elements enable MetaMind to support reliable prediction and long-horizon planning without centralized supervised learning or explicit communication.
\subsection{Limitations}
Despite the promising results of the proposed MetaMind, several limitations remain.

\paragraph{Limited Communication.}
MetaMind depends on observable neighbor trajectories. When task-relevant partners are partially or fully unobservable, e.g., occlusion, misrecognition, sensing dropout, inverse and reflective inference become underconstrained. This degrades belief-goal estimation $\big(\tilde{\boldsymbol{b}}_{t+1}^{(i,j)},\,\hat{\boldsymbol{g}}_{t}^{(i,j)}\big)$ and cannot accurately form collective aggregation $\mathcal{E}_{\varrho}^{i}(\cdot)$. Consequently, prediction and planning reliability cannot be ensured under persistent partial observability of the task-related partners.

\paragraph{Goals Difference.}
The framework assumes a shared task structure across agents despite heterogeneous policies and representations. In adversarial or weakly cooperative settings, particularly with a large or evolving set of unknown goals, this assumption may be violated. Such violations compromise goal identifiability and limit zero-shot generalization beyond the known goal codebook.

\paragraph{Practical Applications.}
Current evaluations are simulation-centric. Real embodied multi-agent deployments introduce additional uncertainties (sensor noise, delays, non-ideal dynamics), resource constraints, safety requirements, and nonstationary conditions. These factors may reduce the fidelity of reflective inference and collective belief formation relative to controlled benchmarks.

\subsection{Future Work}
Several promising directions arise from these limitations: 

\paragraph{Mixed-Communication Mechanism.}
A feasible solution is to extend MetaMind with a \emph{mixed-communication mechanism} that flexibly combines implicit inference from trajectories with lightweight explicit signals when behavioral observations are insufficient. 
Specifically, (a) when neighbors are fully observable, agents rely on Meta-ToM inference as before; (b) under occlusion, recognition failure, or sparse observations, agents trigger minimal communication channels, such as compressed latent codes or uncertainty flags, to exchange partial beliefs without requiring full symbolic protocols; and (c) attention-based gating can adaptively decide when and with whom to communicate, balancing efficiency and robustness. 
This hybrid design preserves MetaMind's decentralized nature while extending its applicability to partially observable multi-agent environments, enabling reliable coordination even in the presence of hidden or transient partners.

\paragraph{Robust Goal Modeling.}
A key direction is to relax the shared-task assumption and pursue open-set goal identification, where agents must infer and act under previously unseen goals and a potentially large goal space. Concretely, we aim to (a) learn compositional and disentangled goal representations that support systematic generalization beyond the goal codebook, (b) incorporate uncertainty-aware inference, e.g., calibrated posteriors or confidence scores, to detect goal novelty and trigger few-shot adaptation, and (c) develop meta-learning and continual discovery procedures that grow the goal repertoire online while preserving past competencies. 
 
\paragraph{Practical Applications.}
A promising direction is to validate MetaMind in real-world embodied multi-agent systems, where uncertainty, partial observability, and physical constraints are unavoidable. 
One feasible pathway is to first deploy MetaMind in controlled robotic platforms, such as multi-robot navigation or cooperative manipulation, where ground-truth sensing allows systematic benchmarking. 
Next, incorporating domain adaptation and sim-to-real transfer techniques can mitigate the gap between simulated and physical environments, ensuring that the reflective inference and goal identifiability mechanisms remain robust under sensor noise and actuation delays. 
Finally, integration with human-AI collaborative settings, such as assistive robotics or mixed-reality coordination tasks, can test the scalability of MetaMind under heterogeneous partners and real-time decision requirements. 
This progressive pipeline provides a practical route to extend the framework beyond simulation and toward embodied intelligence in open-world applications.

\end{document}